\documentclass[twoside]{article}

\usepackage[arxiv]{aistats2019}
\usepackage[round]{natbib}

\usepackage[utf8]{inputenc} %
\usepackage[T1]{fontenc}    %
\usepackage{hyperref}       %
\usepackage{url}            %
\usepackage{booktabs}       %
\usepackage{subfigure}
\usepackage{amsmath,bm,amsthm,bbm,amssymb,mathtools}
\usepackage{amsmath,amssymb,graphicx,url,color,alltt,verbatim,tabularx,stmaryrd,xspace,bm}
\usepackage{hyperref}       %
\usepackage{amsfonts}       %
\usepackage{nicefrac}       %
\usepackage{microtype}      %
\usepackage{tikz}
\usepackage{multicol}
\newcommand*{\Sref}[1]{\hyperref[#1]{\S\ref*{#1}}}
\newcommand*{\egref}[1]{\hyperref[#1]{Example~\ref*{#1}}}
\newcommand*{\lemref}[1]{\hyperref[#1]{Lemma~\ref*{#1}}}
\newcommand*{\thmref}[1]{\hyperref[#1]{Theorem~\ref*{#1}}}
\newcommand*{\appendixref}[1]{\hyperref[#1]{Appendix~\ref*{#1}}}
\newcommand*{\figref}[1]{\hyperref[#1]{Figure~\ref*{#1}}}
\newcommand*{\tabref}[1]{\hyperref[#1]{Table~\ref*{#1}}}

\newcommand{\eg}{\emph{e.g.,} }

\newcommand{\probgt}[2]{\Pr\left[#1 > #2\right]}
\newcommand{\probeq}[2]{\Pr\left[#1 = #2\right]}
\newcommand{\probneq}[2]{\Pr\left[#1 \neq #2\right]}
\newcommand{\probge}[2]{\Pr\left[#1 \geq #2\right]}

\newcommand{\Ex}[1]{\mathbb{E}[#1]}
\newcommand{\Excond}[2]{\mathbb{E}\left[#1\ |\ #2\right]}
\newcommand{\V}[1]{\mathbb{V}\left[#1\right]}

\newcommand{\eqdef}{\mathrel{:=}}

\newcommand{\dir}{\tt Dirichlet}
\newcommand{\cat}{\tt Categorical}

\newcommand{\sketchkern}{T}

\definecolor{StringRed}{rgb}{.637,0.082,0.082}
\definecolor{CommentGreen}{rgb}{0.0,0.55,0.3}
\definecolor{KeywordBlue}{rgb}{0.0,0.3,0.55}
\definecolor{LinkColor}{rgb}{0.55,0.0,0.3}
\definecolor{CiteColor}{rgb}{0.55,0.0,0.3}
\definecolor{HighlightColor}{rgb}{0.0,0.0,0.0}

\hypersetup{%
  linktocpage=true, pdfstartview=FitV,
  breaklinks=true, pageanchor=true, pdfpagemode=UseOutlines,
  plainpages=false, bookmarksnumbered, bookmarksopen=true, bookmarksopenlevel=3,
  hypertexnames=true, pdfhighlight=/O,
  colorlinks=true,linkcolor=LinkColor,citecolor=CiteColor,
  urlcolor=LinkColor
}

\newtheorem{theorem}{Theorem}[section]

\newtheorem{lemma}[theorem]{Lemma}

\newcommand{\joe}[1]{\textcolor{red}{JDT: #1}}

\newcommand{\MincounterAdd}[3]{M'_{#1, #2}(#3)}

\usetikzlibrary{positioning}
\usetikzlibrary{shapes}
\usetikzlibrary{matrix}
\usetikzlibrary{arrows}
\usetikzlibrary{calc}
\usetikzlibrary{arrows.meta}
\usetikzlibrary{patterns}
\pgfdeclarepatternformonly{soft crosshatch}{\pgfqpoint{-1pt}{-1pt}}{\pgfqpoint{4pt}{4pt}}{\pgfqpoint{3pt}{3pt}}%
{
  \pgfsetstrokeopacity{0.3}
  \pgfsetlinewidth{0.4pt}
  \pgfpathmoveto{\pgfqpoint{3.1pt}{0pt}}
  \pgfpathlineto{\pgfqpoint{0pt}{3.1pt}}
  \pgfpathmoveto{\pgfqpoint{0pt}{0pt}}
  \pgfpathlineto{\pgfqpoint{3.1pt}{3.1pt}}
  \pgfusepath{stroke}
}

\usepackage{algorithm,algorithmicx}
\usepackage{algpseudocode}

\algblockdefx{Begin}{End}{{\bf begin}}{{\bf end}}
\newcommand*\Assign[2]{\State #1 $\gets$ #2}
\newcommand*\WriteAs[2]{\State {\bf write} #1 {\bf as} #2}
\newcommand*\Increment[1]{\State {\bf increment} #1}
\newcommand*\Lock[1]{\State {\bf lock} #1}
\newcommand*\Unlock[1]{\State {\bf unlock} #1}

\newcommand*\IfThenAssign[3]{\State {\bf if} #1 {\bf then} #2 $\gets$ #3}

\newcommand*\ForInRange[2]{\ForAll{$0 \leq \hbox{#1} < \hbox{#2}$}}
\newcommand*\LocalArray[1]{\State {\bf local array} #1}

\newcommand*\ClearArray[1]{\State {\bf clear array} #1 \Comment{Set every element to 0}}
\newcommand*\InitializeArray[1]{\State {\bf initialize array} #1 \Comment{Randomly chosen distributions}}
\newcommand*\Bind[2]{\State {\bf let} #1 $\gets$ #2}
\newcommand*\Remark[1]{\State \(\triangleright\) #1}

\algdef{SE}[IF]{IfThenAssignBlock}{EndIf}[3]{{\bf if} #1 {\bf then} #2 $\gets$ #3}{\algorithmicend\ \algorithmicif}
\algdef{C}[IF]{IF}{ElseAssign}[2]{{\bf else} #1 $\gets$ #2}
\algdef{C}[IF]{IF}{ElseOverflowError}{{\bf else} overflow error}
\algdef{C}[IF]{IF}{ElseOverflowErrorWith}[1]{{\bf else} overflow error: #1}

\algtext*{EndIf}  %
\algtext*{EndProcedure}  %

\newbox\myalgbox

\newenvironment{myalgorithmic}{\begin{flushleft}\begin{algorithmic}[1]}{\end{algorithmic}\end{flushleft}}

\algrenewcommand\algorithmicindent{1em}

\newcommand{\tpd}{\mathit{tpd}}
\newcommand{\wpt}{\mathit{wpt}}
\newcommand{\wt}{\mathit{wt}}
\newcommand\cut[1]{}

\newcommand{\stateexact}{S_{\textsf{e}}}

\newcommand{\stationary}{\mu}

\newcommand{\smat}{c}
\newcommand{\rsmat}{C}
\newcommand{\rsrest}{Z}

\newcommand{\srest}{z}

\newcommand{\resttype}{Y}
\newcommand{\kernel}{K}

\newcommand{\kernelpre}{\kernel_{\textsf{pre}}}
\newcommand{\kernelpost}{\kernel_{\textsf{post}}}
\newcommand{\kcomp}{\cdot}
\newcommand{\weaklim}{\Rightarrow}
\newcommand{\plim}{\xrightarrow[]{\textsf{p}}}

\newcommand{\norm}[1]{\lVert #1 \rVert}

\usepackage{thmtools}
\makeatletter
\declaretheoremstyle[%
  spaceabove=-5pt,%
  spacebelow=6pt,%
  headfont=\normalfont\itshape,%
  postheadspace=1em,%
  qed=\qedsymbol%
]{mystyle} 

\declaretheorem[name={Proof},style=mystyle,unnumbered,
]{proof}

\begin{document}

\twocolumn[

\aistatstitle{Sketching for Latent Dirichlet-Categorical Models}

\ifdefined\isanon
  \aistatsauthor{Anonymous Authors}
  \aistatsaddress{ }
\else
  \aistatsauthor{
    Joseph Tassarotti
    \And
    Jean-Baptiste Tristan 
    \And
    Michael Wick
  }
  \aistatsaddress{ Carnegie Mellon University \And Oracle Labs \And Oracle Labs }
\fi
]

\begin{abstract}

Recent work has explored transforming data sets into smaller, approximate
summaries in order to scale Bayesian inference.  We examine a related problem in
which the parameters of a Bayesian model are very large and expensive to store
in memory, and propose more compact representations of parameter values that can
be used during inference.  We focus on a class of graphical models that we refer
to as latent Dirichlet-Categorical models, and show how a combination of two
sketching algorithms known as count-min sketch and approximate counters provide
an efficient representation for them. We show that this sketch combination --
which, despite having been used before in NLP applications, has not been
previously analyzed -- enjoys desirable properties. We prove that for this class
of models, when the sketches are used during Markov Chain Monte Carlo inference,
the equilibrium of sketched MCMC converges to that of the exact chain as sketch
parameters are tuned to reduce the error rate.

\end{abstract}

\section{Introduction}
The development of \emph{scalable Bayesian inference}
techniques~\citep{AngelinoJA16} has been the subject of much recent work. A
number of these techniques introduce some degree of approximation into
inference.

This approximation may arise by altering the inference algorithm. For example,
in ``noisy'' Metropolis Hastings algorithms, acceptance ratios are
perturbed because the likelihood function is either simplified or evaluated on
a random subset of data in each iteration~\citep{Rosenthal2017, Alquier, Pillai,
BardenetDH14}. Similarly, asynchronous Gibbs sampling~\citep{SaRO16} violates some
strict sequential dependencies in normal Gibbs sampling in order to avoid
synchronization costs in the distributed or concurrent setting.

Other approaches transform the original large
data set into a smaller representation, on which traditional inference
algorithms can then be efficiently run. \citet{HugginsCB16} compute
a weighted subset of the original data, called
a \emph{coreset}. \citet{GeppertIMQS17} consider Bayesian regression with $n$
data points each of dimension $d$, and apply a random projection to
shrink the original $\mathbb{R}^{n \times d}$ data set down to
$\mathbb{R}^{k \times d}$ for $k < n$. An advantage of these kinds of
transformations is that by shrinking the size of the data, it becomes more
feasible to fit the transformed data set entirely in memory.

The transformations described in the previous paragraph reduce the number of
data points under consideration, but preserve the \emph{dimension} of each data
point, and thus the number of parameters in the model.  However, in many
Bayesian mixed membership models, the number of parameters themselves can also
become extremely large when working with large data sets, and storing these parameters poses
a barrier to scalability.

In this paper, we consider an approximation to address this issue for
what we call latent Dirichlet-Categorical models, in which there are many latent categorical variables whose distributions are sampled from Dirichlets.
This is a fairly general pattern that
can be found as a basic building block of many Bayesian models used in NLP (\eg clustering of discrete data, topic models like LDA, hidden Markov models). 
The most representative example, which we will use throughout this paper, is the following:
\begin{align}
z_i & \sim \cat(\tau) && i \in [N]\\
\theta_i & \sim \dir(\alpha) && i \in [K]\\
x_i & \sim \cat(\theta_{z_i}) && i \in [N]
\end{align}
Here, $\tau$ is some fixed hyper-parameter of dimension $K$ and $\alpha$ is a scalar value.  
We assume that the dimension of the Dirichlet distribution is $V$, a value we refer to as the ``vocabulary size''. Each random variable $x_i$ can take one of $V$ different values,
which we refer to as ``data types'' (e.g., words in latent Dirichlet allocation).

To do Gibbs sampling for a model in which such a pattern occurs, we generally
need to compute a certain matrix $c$ of dimension $K \times V$. Each row of this
matrix tracks the frequency of occurrence of some data type within one of the
components of the model. In general, this matrix can be quite large, and in some
cases we may not even know the exact value of $V$ a priori (\eg consider the
streaming setting where we may encounter new words during inference), making it
costly to store these counts. Moreover, if we do distributed inference by
dividing the data into subsets, each compute node may need to store this entire
large matrix, which reduces the amount of data each node can store in memory and
adds communication overhead. Although $c$ is often sparse, using a sparse or
dynamic representation instead of a fixed array makes updates and queries
slower, and adds further overhead when merging distributed representations.

We propose to address these problems by using \emph{sketch} algorithms to store
compressed representations of these matrices.  These algorithms give approximate
answers to certain queries about data streams while using far less space than
algorithms that give exact answers. For example, the \emph{count-min sketch}
(CM)~\citep{CMsketch} can be used to estimate the frequency of items in a data
set without having to use space proportional to the number of distinct items,
and \emph{approximate counters} \citep{Morris,Flajolet} can store very large
counts with sublogarithmic number of bits.  These algorithms have parameters
that can be tuned to trade between estimation error and space usage. 
Because many natural language processing tasks involve computing
estimates of say, the frequency of a word in a corpus, there has been
obvious prior interest in using these sketching algorithms for (non-Bayesian) NLP when
dealing with very large data sets~\citep{TOMB,GoyalD11,DurmeL09}.

We propose representing the matrix $c$ above using a combination of count-min sketch and approximate counters. It is not clear a priori what effect this would
have on the MCMC algorithm. On the one hand, it is plausible that if
the sketch parameters are set so that estimation error is small
enough, MCMC will still converge to some equilibrium distribution that
is close to the equilibrium distribution of the exact non-sketched
version. On the other hand, we might be concerned that even small
estimation errors within each iteration of the sampler would compound,
causing the equilibrium distribution to be very far from that of the
non-sketched algorithm.

In this paper, we resolve these issues both theoretically and
empirically. We prove results showing that under fairly general
conditions, as the parameters of sketches are tuned to decrease the
error rate, the equilibrium distributions of sketched chains converge
to that of the exact chain.  Then, we show that when the combined
sketch is used with a highly scalable MCMC algorithm for LDA, we
can obtain model quality comparable to that of
the non-sketched version while using much less space.

\paragraph{Contribution} 

\begin{enumerate}
\item We explain how the count-min sketch algorithm and approximate counters can be used
to sketch the sufficient statistics of models that contain latent Dirichlet-Categorical subgraphs (section \Sref{Solution}). We then provide an analysis of a combined count-min sketch/approximate counter data structure which provides the benefits of both (section \Sref{Analysis}).
\item We then prove that when the combined sketch is used in an MCMC algorithm, as the parameters of the sketch are tuned to reduce error rates, the equilibrium distributions of sketched chains converge to that of the non-sketched version (section \Sref{Convergence}).
\item We complement these theoretical results with experimental evidence confirming that learning works despite approximations introduced by the sketches (section \Sref{Evaluation}).
\end{enumerate}

\section{Sketching for Latent Dirichlet-Categorical Models}\label{Solution}

As described in the introduction, MCMC algorithms for models involving
Dirichlet-Categorical distributions usually require tabulating
statistics about the current assignments of items to categories (\eg
the words per topic in LDA).
There are two reasons why maintaining this matrix of counts can
be expensive. First, the dimensions of the matrix can be large
-- the dimensions are often proportional to the number of unique words
in the corpus. Second, the values in the matrix can also be
large, so that tracking them using small sized integers can
potentially lead to overflow.

Sketching algorithms can be used to address these problems, providing
compact fixed-size representations of these counts that use far less
memory than a dense array. We start by explaining two widely used
sketches, and then in the next section discuss how they can be
combined.

\subsection{Sketch 1: count-min sketch}

To deal with the fact that the matrix of counts is of large dimension, we can use
count-min~(CM)~sketches~\citep{CMsketch} instead of dense arrays. A CM sketch
$\mathcal{C}$ of dimension $l \times w$ is represented as an $l \times
w$ matrix of integers, initialized at 0, and supports two operations:
\texttt{update} and \texttt{query}.
The CM sketch makes
use of $l$ different 2-universal hash functions of range $w$ that we
denote by $h_1, \dots, h_l$.  The \texttt{update}$(x)$ operation
adjusts the CM sketch to reflect an increment to the frequency of some
value $x$, and is done by incrementing the matrix at locations
$\mathcal{C}_{i,h_i(x)}$ for $i \in [1,l]$.  The \texttt{query}($x$)
operation\footnote{
 Other query rules can be used,
 such as the count-mean-min~\citep{CountMeanMin} rule.
 However, \citet{GoyalDC12} suggest that conventional
 CM sketch has better average error for queries of mid to high frequency keys in
 NLP tasks. Therefore, we will focus on the
 standard CM estimator.}
returns an estimate of the frequency of value $x$ and is
computed by $\min_i \mathcal{C}_{i,h_i(x)}$.

It is useful to think of a value $C_{a,b}$ in the matrix as a random
variable. In general, when we study an arbitrary value, say $x$, we
need not worry about where it is located in row $i$ and refer to
$\mathcal{C}_{i,h_i(x)}$ simply as $Z_i$, and write
$Q(x) \eqdef \min_i(Z_i)$ for the result when querying $x$. Note that
$Z_i$ equals the true number of occurrences of $x$, written $f_x$,
plus the counts of other keys whose hashes are identical to that of
$x$. CM sketches have several interesting properties, some of
which we summarize here (see \cite{LectureCountMin} for a good
expository account). Let $N$ be the total number of increments to the
CM sketch. Then, each $Z_i$ is a biased estimator, in that:
\begin{align}
\Ex{Z_i} = f_x + \frac{N - f_x}{w}
\end{align}
However, by adjusting the parameters $l$ and $w$, we can bound the probability of large overestimation. In particular, by taking $w = \frac{k}{\epsilon}$ one can bound the offset of a query as
\begin{align}
\probge{Q(x)}{f_x + \epsilon N} \leq \frac{1}{k^l}
\end{align}
A nice property of CM sketches is that they can be used in parallel:
we can split a data stream up, derive a sketch for each piece, and
then merge the sketches together simply by adding the entries in the
different sketches together componentwise.

We want to replace the \emph{matrix} of counts $c$ in a Dirichlet-Categorical model with sketches. There is some
flexibility in how this is done. The simplest thing is to replace the entire
matrix with a single sketch (so that the keys are the indices into the matrix).
Alternatively, we can divide the matrix into sub-matrices, and use a
sketch for each sub-matrix. In the setting of Dirichlet-Categorical models,
each row of $c$ corresponds to the counts for data types within
one component of the model (\eg counts of words for a given topic in LDA),
so it is natural to use a sketch per row.

\cut{
A first thought that should come to mind is that maybe we should not use a CM sketch in the first place but rather a variant such as count-mean where
instead of returning the value of the smaller counter, we return their average value. But unfortunately, such an approach does not work because the 
count-mean estimator might return negative values. We could of course solve such a problem by returning the maximum between the estimate and 0, but using 
such an operations leads to as much complexity as working with a CM sketch. Last but not least, we ran a significant number of experiments with
a count-mean sketch that showed that such an approach does not work in practice.}

\subsection{Sketch 2: approximate counting}
\label{sec:approximate-counters}
In order to represent large counts without the memory costs of using a
large number of bytes, we can employ approximate counters~\citep{Morris}.  An approximate
counter $X$ of base $b$ is represented by an integer (potentially only
a few bits) initialized at 0, and supports two operations: increment
and read.  We write $X_n$ to denote a counter that has been
incremented $n$ times. The increment operation is randomized and
defined as:
\begin{align}
&\Pr(X_{n+1} = k+1 \ |\ X_n = k) = b^{-k} \\
&\Pr(X_{n+1} = k\ |\ X_n = k) = 1 - b^{-k}
  \label{eqn:counter-transition}
\end{align}
Reading a counter $X$ is written as $\phi(X)$ and defined as $\phi(X) = (b^{X} - 1)/(b -1)$.
Approximate counters are unbiased, and their variance can be controlled by adjusting $b$: %
\begin{align}
\Ex{\phi(X_n)} = n \qquad 
\V{\phi(X_n)} = \frac{b-1}{2} (n^2 - n)
\end{align}
Using approximate counters as part of inference for Dirichlet-Categorical models is very simple: instead of representing the matrix $c$ as an array of integers, we instead use an array of approximate counters.

\cut{
\joe{I think we should cut this para. I think we originally wanted this to tout how SSCA works well with this stuff, but now it does not seem necessary.}
Using such counters places a key restriction on what the inference algorithm can be. For example, we would not be able to use
such counters if we were to collapse (integrate) the parameters of our model, as this would give us a Gibbs sampler where we also need to decrement
the sufficient statistics counters, an operation that is not supported by approximate counters.
}

\section{Combined Sketching: Alternatives and Analysis}\label{Analysis}

\newcommand{\param}{\psi}

The problems addressed by the sketches described in the previous
section are complementary: CM sketches replace a large matrix with a
much smaller set of arrays; but by coalescing increments for distinct
items, CM sketches need to potentially store larger counts to avoid
overflows, a problem which is resolved with approximate counting. Therefore,
it is natural to consider how to combine the two sketching algorithms together.

\subsection{Combination 1: Independent Counters}
\label{sec:comb1}

The simplest way to combine the CM sketch with approximate counters is
to replace each exact counter in the CM sketch with an approximate
counter; then when incrementing a key in the sketch, we independently
increment each of the counters it corresponds to. Moreover, because
there are ways to efficiently add together two approximate
counters~\citep{Adding}, we can similarly merge together multiple
copies of these sketches by once again adding their entries together componentwise.

When we combine the CM sketch and the approximate counters together in
this way, the errors introduced by these two kinds of algorithms interact. It is
challenging to give a precise analysis of the error rate of the
combined structure. However, it is still the case that we can tweak
the parameters of the sketch to make the error rate arbitrarily low.

To make this precise, note that we now have three parameters to tune:
$b$, the base of the approximate counters, $l$ the number of hashes,
and $w$, the range of the hashes.  Given a parameter triple $\param =
(b, l, w)$, write $Q_\param(x)$ for the estimate of key $x$ from a
sketch using these parameters. Then, given a sequence $\param_n =
(b_n, l_n, w_n)$ of parameters, we can ask what happens to the
sequence of estimates $Q_{\param_{n}}(x)$ when we use the sketches on the
same fixed data set:

 \begin{theorem}
 \label{thm:convprob}
 Let $\param_{n} = (b_n, l_n, w_n)$. Suppose $b_n \rightarrow 1$, $w_n \to \infty$ and there exists some $L$ such that $1 \leq l_n \leq L$ for all $n$. Then for all $x$, $Q_{\param_n}(x)$ converges in probability to $f_x$ as $n \to \infty$. 
 \end{theorem}

See \appendixref{app-consistency} in the supplementary material for the full proof.
This result shows that for appropriate sequences $\param_n$ of parameters, the 
estimator $Q_{\param_{n}}(x)$ is consistent. We call a sequence $\param_n$ satisfying the conditions of \thmref{thm:convprob} a \emph{consistent sequence} of parameters.

For our application, we are replacing a matrix of counts with a collection of
sketches for each row, so we want to know not just about the behavior of the
estimate of a single key in one of these sketches, but about the estimates for
all keys across all sketches.
Formally, let $\smat$ be a $K \times V$ dimensional matrix of counts.
Consider a collection of $K$ sketches, each with parameters $\param$, where for each key $v$, we insert $v$ with frequency $\smat_{k, v}$ into the $k$th sketch.
then we write $Q_\param(\smat)$ for the random $K \times V$ matrix giving the
estimates of all the keys in each sketch. Because convergence in probability of a
random vector follows from convergence of each of the components, the
above implies:

 \begin{theorem}
 \label{thm:convprob-vector}
 If $\param_{n}$ is a consistent sequence, then $Q_{\param_n}(\smat)$ converges in probability to $\smat$.
 \end{theorem}

Finally, we have been describing the situation where the keys are inserted with
some deterministic frequency and the only source of randomness is in the hashing
of the keys and the increments of the approximate counter. However, it is natural to consider the case where the frequency of
the keys is randomized as well. To do so, we define the Markov kernel\footnote{Throughout, we assume that all topological spaces are endowed with their Borel $\sigma$-algebras, and omit writing these $\sigma$-algebras.}
$\sketchkern_{\param}$ from $\mathbb{N}^{K \times V}$ to $\mathbb{R}_{\geq 0}^{K \times V}$,
where for each $\smat$, $\sketchkern_{\param}(\smat, \cdot)$ is the
distribution of the random variable $Q_{\param}(\smat)$ considered above. Then if $\mu$
is a distribution on count matrices, $\mu \sketchkern_{\param}$ gives the distribution
of query estimates returned for the sketched matrix.

\subsection{Combination 2: Correlated Counters}

Even though the results above show that the approximation error of the
combined sketch can still be made arbitrarily small, it does not
provide a non-asymptotic bound on the error. Indeed, one issue with
this combination is that the traditional estimation rule for the
sketch relies on the fact that for a key $x$, each of the cells
$Z_1, \dots, Z_l$ corresponding to $x$ is at least as large as $f_x$,
the true frequency of $x$. Therefore, the minimum is the closest
estimate of the count.
But when we instead use approximate counters, it is possible for
each counter's estimate to be smaller than $f_x$, so taking the minimum may cause underestimation.

This underestimation rules out using the so-called \emph{conservative update}
rule~\citep{EstanV02}, a technique which can be used to reduce bias of normal CM
sketches. When using conservative update with a regular CM sketch, to increment
a key $x$, instead of incrementing each of the counters corresponding to $x$, we
first find the minimum value and then only increment counters equal to this
minimum. But because approximate counters can underestimate, this is no longer
justifiable in the combined sketch.

\citet{PitelF15} proposed an alternative way to combine CM sketches
with approximate counters that enables conservative updates.
We call their combination \emph{correlated counters}. \figref{fig:correlated} shows
the increment routine with and without conservative update for
correlated counters. The idea in each is that we generate a single uniform $[0, 1]$
random variable $r$ and use this common $r$ to decide how to
transition each counter value according to the probabilities described in
\Sref{sec:approximate-counters}.

\begin{figure}

\begin{myalgorithmic}
\Procedure{{incr-correlated}}{$C$, $x$}
  \Bind{$r$}{Uniform(0, 1)}
  \For{$i$ from 0 to $l$}
     \Bind{$v$}{$C$[$i$][$h_i(x)$]}
     \IfThenAssign{$r < \frac{1}{b^v}$}{$C$[$i$][$h_i(x)$]}{$v + 1$}
  \EndFor \\
\EndProcedure
\Procedure{{incr-conservative}}{$C$, $x$}
  \Bind{$r$}{Uniform(0, 1)}
  \Bind{$min$}{$\infty$}
  \For{$i$ from 0 to $l$}
     \Bind{$v$}{$C$[$i$][$h_i(x)$]}
     \IfThenAssign{$v < min$}{$min$}{$v$}
  \EndFor
  \\
  \If{$r < \frac{1}{b^{min}}$}
     \For{$i$ from 0 to $l$}

     \If{$C[i][h_i(x)] = min$}
     \Assign{$C[i][h_i(x)]$}{$min + 1$}
     \EndIf
     \EndFor
  \EndIf \\
\EndProcedure
\end{myalgorithmic}
\caption{Increment for CM sketch with correlated approximate counters, with and without conservative update.}
\label{fig:correlated}

\end{figure}

However, \citet{PitelF15} did not give a proof of any statistical
properties of their combination. The following result shows that this
variant avoids the underapproximation bias of the independent counter
version:

\begin{theorem} Let $Q(x)$ be the query result for key $x$ using correlated counters
in a CM sketch with one of the increment procedures from \figref{fig:correlated}. Then,
\[f_x \leq \Ex{Q(x)} \leq f_x + \frac{N - f_x}{w}\]
\label{thm:correlated-bound}
\end{theorem}
\begin{proof} We discuss just the non-conservative update increment procedure, since the proof
is similar for the other case. The upper bound is straightforward.
The lower bound is proved by
exhibiting a coupling~\citep{lindvall2002} between the sketch counters
corresponding to key $x$ and a counter $C$ of base $b$ that will be
incremented exactly $f_x$ times. The coupling is constructed by
induction on $N$, the total number of increments to the
sketch. Throughout, we maintain the invariant that $\phi(C) \leq
Q(x)$; it follows that
$\Ex{\phi(C)} \leq \Ex{Q(x)}$. Since $\Ex{\phi(C)} = f_x$, this will give the desired bound.

In the base case, when $N = 0$, both $\phi(C)$ and $Q(x)$ are
$0$ so the invariant holds trivially. Suppose the invariant holds after the first $k$ increments to the
sketch, and some key $y$ is then incremented. If $x = y$, then we
transition the counter $C$ using the same random uniform variable $r$
that is used to transition the counters $X_1, \dots, X_l$
corresponding to key $x$ in the sketch. There are two cases: either
$r$ is small enough to cause the minimum $X_i$ to increase by 1, or
not. If it is, then since $C \leq \min_i(X_i)$, $r$ is also small
enough to cause $C$ to increase by 1, and so $\phi(C) \leq Q(x)$. If
$\min_i(X_i)$ does not change, but $C$ does, then we must have $C < \min_i(X_i)$ before the
transition; since $C$ can only increase by $1$, we still have $C \leq \min_i(X_i)$ afterward.

If the key $y$ is not equal to $x$, then we leave $C$ as is. 
Since the $X_i$ can only possibly increase while $C$ stays the same,
the invariant holds. Finally, after all $N$ increments
have been performed, $C$ will have received $f_x$ increments, so that
$\Ex{\phi(C)} = f_x$ because approximate counters are unbiased.
\end{proof}

In \appendixref{sec:micro} we describe various microbenchmarks comparing
the behavior of the different ways of combining the two sketches.
 
\section{Asymptotic Convergence}\label{Convergence}

In the previous section, we explored some of the statistical properties of the
combined sketch. We now turn to the question of the behavior of an MCMC
algorithm when we use these sketches in place of exact counts. More precisely,
suppose we have a Markov chain whose states are tuples of the form $(\smat,
\srest)$, where $\smat$ is a $K \times V$ matrix of counts, and $\srest$ 
is an element of some complete separable metric space
$\resttype$. Now, suppose instead of tabulating $\smat$ in a dense array of
exact counters, we replace each row with a sketch using parameters $\param$.
We can ask whether the resulting sketched chain\footnote{Since approximate
  counters can return floating point estimates of counts, replacing the exact
  counts with sketches only makes sense if the transition kernel for the Markov
  chain can be interpreted when these state components involve floating point
  numbers. But this is usually the case since Bayesian models typically apply
  non-integer smoothing factors to integer counts anyway.} has an equilibrium
distribution, and if so, how it relates to the equilibrium distribution of the
original ``exact'' chain.  As we will see, it is often easy to show that the
sketched chain still has an equilibrium distribution. However, the relationship
between the sketched and exact equilibriums may be quite complicated. Still, a
reasonable property to want is that, if we have a consistent sequence of
parameters $\param_{n}$, and we consider a sequence of sketched chains, where
the $i$th chain uses parameters $\param_i$, then the sequence of equilibrium
distributions will converge to that of the exact chain.

The reason such a property is important is that it provides some justification for how these
sketched approximations would be used in practice. Most likely, one would first test the
algorithm using some set of sketch parameters, and then if the results are
not satisfactory, the parameters could be adjusted to decrease error
rates in exchange for higher computational cost. (Just as, when using standard MCMC
techniques without an a priori bound on mixing times, one can run chains
for longer periods of time if various diagnostics suggest poor results).
Therefore, we would
like to know that asymptotically this approach really would converge
to the behavior of the exact chain.
We will now show that under reasonable conditions, this convergence
does in fact hold.

We assume the state space $\stateexact$ of the original chain is a compact, measurable
subset of $\mathbb{N}^{K \times V} \times Y$. 
We suppose that the transition kernel of the chain can be divided into three
phrases, represented by the composition of kernels $\kernelpre \kcomp \kernel
\kcomp \kernelpost$, where in $\kernel$ the matrix of counts is updated in a way
that depends only on the rest of the state, which
is then modified in $\kernelpre$ and $\kernelpost$ (\eg in a
blocked Gibbs sampler $\kernel$ would correspond to the part of a sweep where
$\smat$ is updated). Moreover, we assume that the transitions $\kernelpre$ and $\kernelpost$
are well-defined on the extended state space $\mathbb{R}_{\geq 0}^{K \times V} \times Y$, where
the counts are replaced with positive reals.
Formally, these conditions mean we assume that there exist Markov kernels
$\kernelpre', \kernelpost': \mathbb{R}_{\geq 0}^{K \times V} \times Y \rightarrow Y$ and
$\kernel : Y \rightarrow \mathbb{N}^{K \times V}$ such that
\begin{align*}
\kernelpre((\smat, \srest), A) &= \int \kernelpre'((\smat, \srest), d\srest') 1_A(\smat, \srest') \\
\kernel((\smat, \srest), A) &= \int \kernel'(\srest, d\smat') 1_A(\smat', \srest) \\
\kernelpost((\smat, \srest), A) &= \int \kernelpost'((\smat, \srest), d\srest') 1_A(\smat, \srest')
\end{align*}
where we write $1_A$ for the indicator function corresponding to a measurable set $A$.
We assume
this chain has a unique stationary distribution $\stationary$. Furthermore, we assume
$\kernelpre$, $\kernel$, and $\kernelpost$ are \emph{Feller
  continuous}, that is, if $s_n \rightarrow s$, then $\kernel(s_n,
\cdot) \weaklim \kernel(s, \cdot)$, and similarly for $\kernelpre$ and
$\kernelpost$, where $\weaklim$ is weak convergence of measures.

Fix a consistent sequence of parameters $\param_n$. For each $n$, we define the sketched Markov
chain with transition kernel $\kernelpre \kcomp \kernel_n \kcomp \kernelpost$, where
$\kernel_n$ is the kernel obtained by replacing the exact matrix of counts used in $\kernel$ with a sketched
matrix with parameters $\param_n$: 
\begin{align*}
&  \kernel_n((\smat, \srest), A) \\ &= \int \kernel'(\srest, d\smat') \int \sketchkern_{\param_n}(\smat', d\smat'') 1_A(\smat'', \srest)
\end{align*}

(recall that $\sketchkern_{\param_n}$ is the kernel induced by the combined sketching algorithm, as described in
\Sref{sec:comb1}). We assume that the set $S$
containing the union of the states of the exact chain and the sketched
chains is some compact measurable subset of $\mathbb{R}_{\geq 0}^{K \times V} \times Y$.
Assuming that each $\kernelpre \kcomp \kernel_n \kcomp \kernelpost$ has a stationary distribution
$\stationary_n$, we will show that they converge weakly to
$\stationary$. We use the following general result of Karr:

\begin{theorem}[{\citet[Theorems 4 and 6]{Karr75}}]
  Let $E$ be a complete separable metric space with Borel sigma algebra $\Sigma$. Let $K$ and $K_1, K_2, \dots$ be Markov kernels on $(E, \Sigma)$. Suppose $K$ has a unique stationary distribution $\mu$ and $K_1, \dots$ have stationary distributions $\mu_1, \dots$. \\ Assume the following hold
  \begin{enumerate}
    \item for all $s$, $\{K_n(s, \cdot)\}_{n}$ is tight, and
    \item $s_n \rightarrow s$ implies $K_n(s_n, \cdot) \weaklim K(s, \cdot)$.
  \end{enumerate}
  Then $\mu_n \weaklim \mu$.
\end{theorem}

We now show that the assumptions of this theorem hold for our chains. The first condition is straightforward:
\begin{lemma}
For all $x$, the family of measures $\{(\kernelpre \kcomp \kernel_n \kcomp \kernelpost)(x, \cdot)\}_{n}$ is tight.
\end{lemma}
\begin{proof}
This follows immediately from the assumption that the set of states $S$ is a compact measurable set.
\end{proof}

To establish the second condition, we start with the following:
\begin{lemma}
If $s_n \rightarrow s$, then $(\kernelpre \kcomp \kernel_n)(s_n, \cdot) \weaklim (\kernelpre \kcomp \kernel)(s, \cdot)$.
\end{lemma}
\begin{proof}
  To match up with the results in \Sref{Analysis}, it is helpful to
  rephrase this as a question of convergence of distribution of random
  variables with appropriate laws. By assumption $\kernelpre \kcomp \kernel$ is Feller continuous, so we know that $(\kernelpre \kcomp \kernel)(s_n,
  \cdot) \weaklim (\kernelpre \kcomp \kernel)(s, \cdot)$, hence by Skorokhod's representation
  theorem, there exists random matrices $\rsmat, \rsmat_1, \dots,$ and
  random $Y$-elements  $\rsrest, \rsrest_1, \dots$
  such that the law of $(\rsmat_n,
  \rsrest_n)$ is $(\kernelpre \kcomp \kernel)(s_n, \cdot)$, that of $(\rsmat, \rsrest)$ is
  $(\kernelpre \kcomp \kernel)(s, \cdot)$, and $(\rsmat_n, \rsrest_n) \plim (\rsmat, \rsrest)$.
  Then the distribution of $(Q_{\param_{n}}(\rsmat_n), \rsrest_n)$
  is that of $(\kernelpre \kcomp \kernel_n)(s_n, \cdot)$, so it suffices to show that
  $Q_{\param_{n}}(\rsmat_{n}) \plim \rsmat$.

  Fix $\delta, \epsilon > 0$.
  Let $U$ be the union of the supports of each $\rsmat_{n}$. Then $U$ consists of
  integer matrices lying in some compact subset of real vectors (since
  $S$ is compact and the counts returned by $\kernel$ are exact
  integers), so $U$ is finite. Moreover, by \thmref{thm:convprob-vector} we know that for
  all $\smat$, there exists $n_{\smat}$ such that for all $n > n_{\smat}$,
  $ \Pr\left[\norm{Q_{\param_{n}}(\rsmat_{n}) - \smat} > \epsilon/2\ \mid \ \rsmat_{n} = \smat\right] < \delta/2 $.
  Let $m_1$ be the maximum of the $n_{\smat}$ for $\smat \in U$.
  We also know that there exists $m_2$ such that for all $n > m_2$, 
  $ \Pr\left[\norm{\rsmat_{n} - \rsmat} > \epsilon/2\right] < \delta/2 $.
  For $n > \max(m_1, m_2)$, we then have
  $\Pr\left[\norm{Q_{\param_{n}}(\rsmat_{n}) - \rsmat} > \epsilon\right] < \delta$.
\end{proof}

Continuity of $\kernelpost$ then gives us:
\begin{lemma}
If $s_n \rightarrow s$, then $(\kernelpre \kcomp \kernel_n \kcomp \kernelpost)(s_n, \cdot) \weaklim (\kernelpre \kcomp \kernel \kcomp \kernelpost)(s, \cdot)$.
\end{lemma}

Thus by Karr's theorem we conclude:
\begin{theorem} $\stationary_{n} \weaklim \stationary$.  \end{theorem}

In the above, we have assumed that there is a single sketched matrix of counts,
and that each row of the matrix uses the same sketch parameters. However, the
argument can be generalized to the case where there are several sketched
matrices with different parameters.  We now explain how this result can be
applied to some Dirichlet-Categorical models:

\paragraph{Example 1: SEM for LDA.}
When using stochastic expectation maximization (SEM) for the LDA topic
model~\citep{LDA}, the states of the Markov chain are matrices $wpt$
and $tpd$ giving the words per topic and topic per document
counts. Within each round, estimates of the corresponding topic and
document distributions $\theta$ and $\phi$ are computed from smoothed
versions of these counts; new topic assignments are sampled according
to this distribution, and the counts $wpt$ and $tpd$ are updated. We
can replace the rows of either $wpt$ or $tpd$ with sketches. In this
case $\kernelpre$ and $\kernelpost$ are the identity, and the Feller
continuity of $\kernel$ follows from the fact that the estimates of
$\theta$ and $\phi$ are continuous functions of the $\wpt$ and $\tpd$
counts. Compactness of the state space is a consequence of the fact
that the set of documents (and hence maximum counts) are finite, and the maximum counter base is bounded.
Finally, the sketched kernels still have unique stationary
distributions because the smoothing of the $\theta$ and $\phi$
estimates guarantees that if a state is representable in the sketched
chain, we can transition to it in a single step from any other state.

\paragraph{Example 2: Gibbs for Pachinko Allocation.}
The Pachinko Allocation Model (PAM)~\citep{li06pachinko} is a generalization
of LDA in which there is a hierarchy of topics with a directed
acyclic structure connecting them. A blocked Gibbs sampler for this
model can be implemented by first conditioning on topic distributions
and sampling topic assignments for words, then conditioning on these
topic assignments to update topic distributions -- in the latter phase,
one needs counts of the occurrences of words in the different topics
and subtopics which can be collected using sketches. Since the priors
for sampling topics based on these counts are smoothed, the sketched
chains once again have unique stationary distributions for the same reason as in LDA.

\section{Experimental Evaluation}\label{Evaluation}

We now examine the empirical performance of using these sketches. We
implemented a sketched version of SCA~\citep{SCA}, an optimized form
of SEM which is used in state of the art scalable topic
models~\citep{SCA, ZenLDA, WarpLDA, SaberLDA}, and apply it LDA. Full
details of SCA can be found in the appendix.

\paragraph{Setup}  We fit LDA (100 topics,
$\alpha=0.1$, $\beta=0.1$, 291k-word vocabulary after removing rare and
stop-words as is customary) to 6.7 million English Wikipedia articles using 60
iterations of SCA distributed across eight 8-core machines, and measure the
perplexity of the model after each iteration on 10k randomly sampled Reuters
documents.  For all experiments, we report the mean and standard-deviation of
perplexity and timing across three trials.  Example topics from the various
configurations are shown in the appendix.  For more details, see
Appendix~\ref{sec:apx-lda-hardware}.

In this distributed setting, each machine must store a copy of the
word-per-topic ($wpt$) frequency counts, and at the end of an iteration, updated
counts from different machines must be merged.  However, each machine
only needs to store the rows of the topics-per-document matrix ($tpd$)
pertaining to the documents it is processing. Hence, controlling the size of
$wpt$ is more important from a scalability perspective, so we will examine
the effects of sketching $wpt$.

The data set and number of topics we are using for these tests are small enough
that the non-sketched $wpt$ matrix and documents can feasibly fit in each
machine's memory, so sketching is not strictly necessary in this setting. Our
reason for using this data set is to be able to produce baselines
of statistical performance for the non-sketched version to compare against the
sketched versions.

\begin{figure}
\begin{center}
    \includegraphics[width=0.5\textwidth]{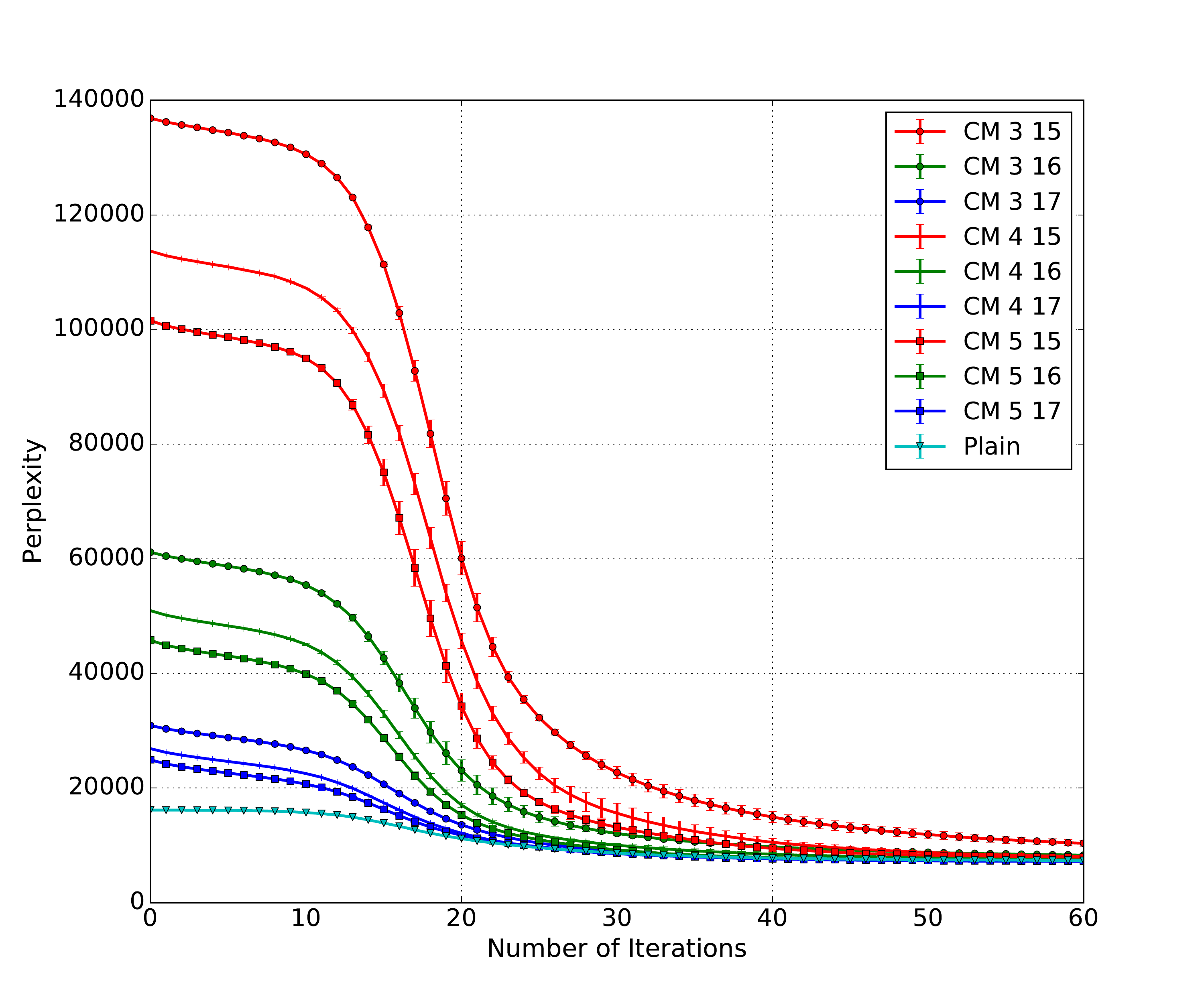}
\end{center}
 \caption{LDA perplexity with count-min sketch.}
\label{fig:body:perplexity:cmsketch}
\end{figure}

\paragraph{Experiment 1: Impact of the CM sketch.} In the first
experiment, we evaluate the results of just using the CM sketch. We replace each
row of the $wpt$ matrix in baseline plain SCA
with a count-min sketch.  We vary the number of hash functions $l\in\{3,4,5\}$
and the bits per hash from $\{15,16,17\}$.  Figure~\ref{fig:body:perplexity:cmsketch} displays
perplexity results for these configurations.
While the more compressive
variants of the sketches start at worse perplexities, by the final iterations,
they converge to similar perplexities as the exact baseline with arrays.
The range of the hash has a much larger effect than the number of hash functions
in the earlier iterations of inference.

Table~\ref{table:body:time:cmsketch} gives timing and space usage (the first row
corresponds to the baseline time and space). Recall that our main interest in
sketching is to reduce space usage. Note that some of the parameter
configurations here use more space than a dense array, so the purpose of
including them is to better understand the statistical and timing effects of the
parameters. Even though the smaller configurations do save space compared to the
baseline, hashing the keys adds time overheads. Again, this is relative to the
ideal case for the baseline, in which the documents and the full $wpt$ matrix
represented as a dense array can fit in main memory.

\begin{table}
\centering
\begin{tabular}{|c|c|c|c|}
\hline
$l$ & $\log_2(w)$ & time (s) & size ($10^5$ bytes) \\
\hline
NA & NA & 12.14 $\pm$ 1.82 & 1164.0\\
\hline
3 & 15 & 22.75  $\pm$ 4.30 & 393.2\\ 
3 & 16 & 23.90  $\pm$ 4.41 & 786.43\\
3 & 17 & 25.32  $\pm$ 4.68 & 1572.9\\
\hline
4 & 15 & 29.70 $\pm$  5.82 & 524.3\\
4 & 16 & 32.75  $\pm$ 6.17 & 1048.6\\
4 & 17 & 33.35  $\pm$ 5.89 & 2097.2\\
\hline
5 & 15 & 37.76  $\pm$ 6.95 & 655.4\\
5 & 16 & 39.71  $\pm$ 7.01 & 1310.7\\
5 & 17 & 42.33  $\pm$ 7.75 & 2621.4\\
\hline
\end{tabular}
\caption{Time per iteration and size of $wpt$ representation for LDA with CM sketch. The first row
gives non-sketched baseline.  4-byte integers are used to store entries in the
dense matrix and sketches.}
\label{table:body:time:cmsketch}
\end{table}

\paragraph{Experiment 2: Combined Sketches} 

\begin{figure}
\begin{center}
    \includegraphics[width=0.5\textwidth]{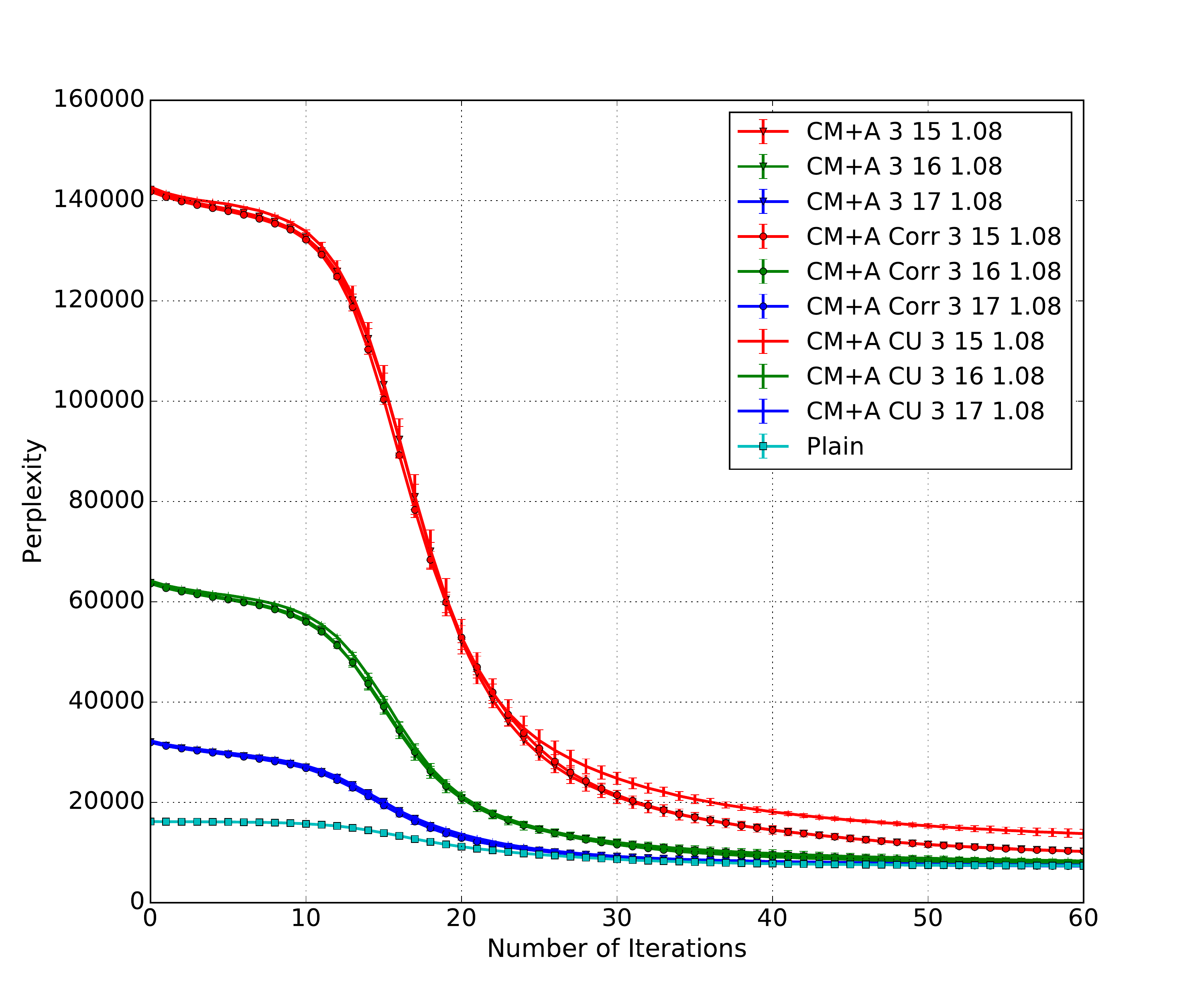}
\end{center}
 \caption{LDA perplexity with combined sketches.}
\label{fig:body:perplexity:combined}
\end{figure}

For the next experiment (Figure~\ref{fig:body:perplexity:combined}), we use the
three variants of combined sketches with approximate counters described in
Section~\ref{Analysis} (sketch with independent counters (\texttt{CM+A}), sketch
with correlated counters (\texttt{CM+A Corr}), and sketch with correlated
counters and the conservative update rule (\texttt{CM+A CU})). We use 1-byte
base-1.08 approximate counters in order to represent a similar range as a 4-byte
integer (but using 1/4 the memory). Given the results of the previous
experiment, we just consider the case where $3$ hash functions are used. In this
particular benchmark, we do not see a large difference in perplexity between the
various update rules, which again converge reasonably close to the perplexity of
the baseline.

Table~\ref{table:body:time:combined} gives timing and space usage for the combined sketches
using the independent counter update rule.
Each iteration runs faster than when just using the CM sketch with similar parameters.
This is because the combined sketches are a quarter of the size of the CM sketch, so there
is less communication complexity involved in sending the representation to other machines.

\begin{table}
\centering
\begin{tabular}{|c|c|c|c|}
\hline
$l$ & $\log_2(w)$ & time (s) & size ($10^5$ bytes) \\
\hline
NA & NA & 12.14 $\pm$ 1.82 & 1164.0\\
\hline
3 & 15 & 12.58 $\pm$ 2.00 & 98.3 \\
3 & 16 & 17.57 $\pm$ 2.78 & 196.6 \\
3 & 17 & 22.69 $\pm$ 3.72 & 393.22 \\
\hline
\end{tabular}
\caption{Time per iteration results for LDA with combined sketch using 1-byte, base 1.08 independent independent counters. Timing for other update rules are similar.}
\label{table:body:time:combined}
\end{table}

\paragraph{}We explore a more comprehensive set of sketch and counter parameter effects on perplexity in
Appendix~\ref{sec:apx-all-params}, run time in
Appendix~\ref{sec:apx-timing}, and example topics in Appendix~\ref{sec:apx-topics}.

\section{Conclusion}

As machine learning models grow in complexity and datasets grow in size, it is becoming
more and more common to use sketching algorithms to represent the data
structures of learning algorithms. When used with MCMC algorithms, a primary question is what effect sketching will have on equilibrium distributions.
In this paper we analyzed sketching algorithms that are commonly used to scale non-Bayesian NLP applications and proved that their use in various MCMC algorithms is justified by showing that sketch parameters can be tuned to reduce the distance between sketched and exact equilibrium distributions.

\bibliography{ref}
\bibliographystyle{plainnat}
\clearpage
\appendix

\section{Consistency of CM sketch with Approximate Counters}
\label{app-consistency}
In this section, we prove \thmref{thm:convprob} from \Sref{sec:comb1}:
\newtheorem*{theorem:convprobfull}{Theorem \ref*{thm:convprob}}
 \begin{theorem:convprobfull}
 \label{thm:convprob-full}
 Let $\param_{n} = (b_n, l_n, w_n)$. Suppose $b_n \rightarrow 1$, $w_n \to \infty$ and there exists some $L$ such that $l_n \leq L$ for all $n$. Let $Q_{\param_n}$ be a sketch with these parameters, using either independent counters or correlated counters. Then for all $x$, $Q_{\param_n}(x)$ converges in probability to $f_x$ as $n \to \infty$. 
 \end{theorem:convprobfull}

 \begin{proof}

Let $N$ be the sum of the frequencies of all keys. Fix some key $x$,
and write $f_x$ for its true frequency. Now, for all
$\delta, \epsilon > 0$ we must show that there exists $n_0$ such that
for $n' > n_0$, $\probgt{|Q_{\param_{n'}}(x) - f_x|}{\epsilon}
< \delta$. Let $X_{n, i}$ be the random variable for the $i$th counter
for key $x$ in sketch $n$, and let $Z_{n, i}$ be the random variable
indicating how many times this counter is incremented.

As in the regular analysis of the count-min sketch, we
have that $\Ex{Z_{n, i}} \leq f_x + \frac{N}{w_n}$ and
$Z_{n, i} - f_x \geq 0$. Thus by Markov's inequality applied to this
difference, we have that $\probgt{Z_{n, i} -
f_x}{1/2} \leq \frac{2N}{w_n}$. Thus for large enough
$w_n$, we can make this probability arbitrarily small. Since
$w_n \rightarrow \infty$, there exists $n_1$ such that for $n > n_1$,
$\probgt{Z_{n, i} - f_x}{1/2} < \frac{\delta}{2L}$.  Since $Z_{n, i}$ and
$f_x$ are both integers, this implies
$\probneq{Z_{n, i}}{f_x} < \frac{\delta}{2L}$.  Now, applying the Chebyshev bound to $X_{n, i}$ we have:
\[
\Pr\left[| X_{n, i} - f_x | > \epsilon \mid Z_{n, i} = f_x\right] \leq \frac{b_n-1}{2 \epsilon^2} (f_x^2 - f_x)
\]
So that once more, for $b_n$ close enough to $1$, this probability will be less than $\frac{\delta}{2L}$, and there exists some $n_2$ such that for all $n > n_2$, $b_n$ will be sufficiently close to $1$. Set $n_0 = \max(n_1, n_2)$, then we have that for $n > n_0$:
\begin{align*}
&\Pr\left[| X_{n, i} - f_x | > \epsilon\right]
\\&\quad= \Pr\left[| X_{n, i} - f_x | > \epsilon \mid Z_{n, i} = f_x\right] \probeq{Z_{n, i}}{f_x}
\\&\quad\qquad{} + \Pr\left[| X_{n, i} - f_x | > \epsilon \mid Z_{n, i} \neq f_x\right] \probneq{Z_{n, i}}{f_x} 
\\&\quad\leq \frac{\delta}{L}
\end{align*}
Thus we have that for such $n$:
\begin{align*}
\Pr\left[| Q_{\param_n}(x) - f_x | > \epsilon\right]
&= \Pr\left[| \min_i X_{n, i} - f_x | > \epsilon\right] \\
&\leq \sum_i^{l_n}  \Pr\left[| X_{n, i} - f_x | > \epsilon\right] \\
&\leq l_n \cdot \frac{\delta}{L}  \\
&\leq \delta
\end{align*}
where the second inequality follows from the union bound. Because we use the union bound, it does not matter whether the $X_{n, i}$ are correlated for different $i$.

 \end{proof}

 The above argument can be adapted to handle the case when we add
 together a finite number of sketches using the \citet{Adding} rule
 for adding approximate counters. The idea is to argue that when the
 base of the counters of the two summands are close to 1, they will
 both with high probability be equal to the ``true'' number of increments
 that have been performed to each, and similarly so will the sum.

\section{Bias from Merging Combined Sketches}

As we have explained in the body of the paper, merging the combined
sketches can be done using the addition routine of
\citet{Adding}. However, if these additions are done independently,
this introduces potential for underestimation even in the case where
correlated counters are used for incrementing within each sketch prior to merging.
In this section we bound the resulting bias.

The expected values of maxima and minima of IID random variables is well studied. In particular,
we have:
 
\begin{theorem}[\citet{Gumbel54}, \citet{Hartley1954}]
Let $Y_1, \dots, Y_l$ be a collection of iid random variables, such
that $\Ex{Y_1} = \mu$ and $\V{Y_1} = \sigma^2$.
\[ \Ex{\max_{i} Y_i} \leq \mu + \sigma \left(\frac{l-1}{\sqrt{2l - 1}}\right) \]
\label{thm:gumbel}
\end{theorem}

Let $X'_n$ be the value of an approximate counter obtained by adding together several independent base $b$ counters that have been incremented a total of $n$ times collectively.

Then, the results of \citet{Adding} show that:
\[ \V{\phi(X'_n)} \leq \frac{b - 1}{2}(n^2 - n) + \frac{1}{-2(b^2 - 4b + 1)} \]
Note that this bound holds regardless of how many counters were added together or what their relative sizes were.

\newcommand{\devbound}{\xi}

As \citet{Adding} point out, for $1 < b \leq 2$, the right summand is
less than $\frac{1}{4}$.  Define $\devbound(n) =
\sqrt{\frac{b-1}{2}(n^2 - n)} + \frac{1}{2}$, then we have that
$\sqrt{\V{\phi(X'_n)}} \leq \devbound(n)$ for $b$ in this range.

\subsection{Merging Independent Counter Sketches}

Let $Y'_1, \dots, Y'_l$ be IID random variables such that each $Y'_i$ is an approximate counter whose value is obtained by adding together several independent base $b$ counters that have been incremented a total of $n$ times collectively. Let $\MincounterAdd{b}{l}{n}$ be $\min_i \phi(Y'_i)$.

\begin{theorem} If $1 < b \leq 2$, then
\[ \Ex{\MincounterAdd{b}{l}{n}} \geq
  n \left(1 - \left(\sqrt{\frac{b-1}{2}} + \frac{1}{2n}\right)\left(\frac{l-1}{\sqrt{2l - 1}}\right)\right) \]
\end{theorem}
\begin{proof}
Applying \thmref{thm:gumbel}, we obtain:
 \begin{align*}
 &  \Ex{\MincounterAdd{b}{l}{n}} \\
 & \geq n - \devbound(n)\left(\frac{l-1}{\sqrt{2l - 1}}\right) \\
 &\geq n\left(1 - \left(\sqrt{\frac{b - 1}{2}\left(1 - \frac{1}{n}\right)}
         + \frac{1}{2n}\right)\left(\frac{l-1}{\sqrt{2l - 1}}\right)\right) \\
 &\geq n\left(1 - \left(\sqrt{\frac{b - 1}{2}}
         + \frac{1}{2n}\right)\left(\frac{l-1}{\sqrt{2l - 1}}\right)\right)
 \end{align*}
 \end{proof}

Then, if $Q$ is a CM sketch obtained by merging several sketches using independent counters, and $f_x$ is the total frequency of key $x$ across the sketches, then $\Ex{\MincounterAdd{b}{l}{f_x}} \leq \Ex{Q(x)}$

\subsection{Merging Correlated Counter Sketches}

If we use the correlated increment rule of \citet{PitelF15}, then the
counters corresponding to a key within each sketch are not
independent, hence they are not independent in the merged sketch
either, so the results of \thmref{thm:gumbel} do not immediately
apply. However, we can still obtain the same bound as in the independent case, as we will see.

To simplify the notation, we will just assume we are merging only two
sketches; the proof generalizes to merging an arbitrary number of
sketches. Let $f_{x,1}$ be the frequency that $x$ occurs in the data
stream processed by the first sketch, and let $f_{x,2}$ be the
frequency in the second sketch, so that $f_x = f_{x,1} +
f_{x,2}$. Since we seek a lower bound $\Ex{Q(x)}$, and hash collisions
between $x$ and other keys only possibly increase $\Ex{Q(x)}$, it
suffices to bound the case when there are no collisions.

In that case, the correlated increment rule is the same whether we do
conservative update or not, and the values of the counters for key $x$
within sketch $1$ are all equal, and similarly for all the counters
within sketch $2$. Let $Y_1$ and $Y_2$ denote these values. For each $m$ and $n$, if we condition
on $Y_1 = m$ and $Y_2 = n$, then the results of adding the cells for key $x$ are IID,
and so \thmref{thm:gumbel} applies once more, so that we get:

\begin{align}
  & \Excond{Q(x)}{Y_1 = m, Y_2 = n} \\
  & \geq (n + m) - \devbound(n+m)\left(\frac{l-1}{\sqrt{2l - 1}}\right) \\
  & \geq (n + m)\left(1 - \sqrt{\frac{b-1}{2}}\left(\frac{l-1}{\sqrt{2l - 1}}\right)\right) \\
   &\quad - \frac{1}{2}\left(\frac{l-1}{\sqrt{2l - 1}}\right)
\end{align}

By the law of the total expectation, it follows that 
\begin{align}
  &  \Ex{Q(x)} \\
  & \geq (\Ex{Y_1} + \Ex{Y_2})\left(1 - \sqrt{\frac{b-1}{2}}\left(\frac{l-1}{\sqrt{2l - 1}}\right)\right)  \\
  &\quad - \frac{1}{2}\left(\frac{l-1}{\sqrt{2l - 1}}\right) \\
  &= f_x\left(1 - \sqrt{\frac{b-1}{2}}\left(\frac{l-1}{\sqrt{2l - 1}}\right)\right)  \\
  &\quad - \frac{1}{2}\left(\frac{l-1}{\sqrt{2l - 1}}\right) \\
  &= f_x\left(1 - \left(\sqrt{\frac{b-1}{2}} + \frac{1}{2 f_x}\right)\left(\frac{l-1}{\sqrt{2l - 1}}\right)\right) 
\end{align}

\section{LDA}
\label{sec:apx-lda}

\subsection{Inference with SCA}
\label{sec:apx-lda-inference}

SCA for LDA has the following parameters: $I$ is the number of iterations to
perform, $M$ is the number of documents, $V$ is the size of the
vocabulary, $K$ is the number of topics, $N[M]$ is an integer array of
size $M$ that describes the shape of the data $w$, $\alpha$ is a parameter that
controls how concentrated the distributions of topics per documents
should be, $\beta$ is a parameter that controls how concentrated the
distributions of words per topics should be, $w[M][N]$ is ragged array
containing the document data (where subarray $w[m]$ has length
$N[m]$), $\theta[M][K]$ is an $M \times K$ matrix where $\theta[m][k]$
is the probability of topic $k$ in document $m$, and $\phi[V][K]$ is a
$V \times K$ matrix where $\phi[v][k]$ is the probability of word $v$
in topic $k$.  Each element $w[m][n]$ is a nonnegative integer less
than $V$, indicating which word in the vocabulary is at position $n$
in document $m$.  The matrices $\theta$ and $\phi$ are typically
initialized by the caller to randomly chosen distributions of topics
for each document and words for each topic; these same arrays serve to
deliver ``improved'' distributions back to the caller.

The algorithm uses three local data structures to store various
statistics about the model
(lines~\ref{ssca-decl-start}--\ref{ssca-decl-end}): $\tpd[M][K]$
is a $M \times K$ matrix where $\tpd[m][k]$ is the number of times topic
$k$ is used in document $m$, $\wpt[V][K]$ is an $V \times K$ matrix
where $\wpt[v][k]$ is the number of times word $v$ is assigned to
topic $k$, and $\wt[K]$ is an array of size $K$ where $\wt[k]$ is the
total number of time topic $k$ is in use. We actually have two copies
of each of these data structures because the algorithm alternates
between reading one to write in the other, and vice versa.

The SCA algorithm iterates over the data to compute statistics for the
topics (loop starting on line~\ref{ssca-iterate-start} and ending on
line~\ref{ssca-iterate-end}). The output of SCA are the two probability
matrices $\theta$ and $\phi$, which need to be computed in a
post-processing phase that follows the iterative phase. This post-processing phase is similar to the one of a classic collapsed Gibbs sampler. In this
post-processing phase, we compute the $\theta$ and $\phi$
distributions as the means of Dirichlet distributions induced by the
statistics.

In the iterative phase of SCA, the values of $\theta$ and $\phi$,
which are necessary to compute the topic proportions, are computed on
the fly (lines~\ref{ssca-compute-theta} and~\ref{ssca-compute-phi}).
Unlike the Gibbs algorithm, where in each iteration we have a
back-and-forth between two phases, where one reads $\theta$ and $\phi$
in order to update the statistics and the other reads the statistics
in order to update $\theta$ and $\phi$; SCA performs the
back-and-forth between two copies of the statistics. Therefore, the
number of iterations is halved (line~\ref{ssca-iterate-count}), and
each iteration has two subiterations (line~\ref{ssca-iterate-twice}),
one that reads $\tpd[0]$, $\wpt[0]$, and $\wt[0]$ in order to write
$\tpd[1]$, $\wpt[1]$, and $\wt[1]$, then one that reads $\tpd[1]$,
$\wpt[1]$, and $\wt[1]$ in order to write $\tpd[0]$, $\wpt[0]$, and
$\wt[0]$.  

\paragraph{Sketching and hasing}
\label{sec:apx-lda-implementation}
Modifying SCA to support feature hashing and the CM sketch is fairly simple.
The read of $\wpt$ on line \ref{ssca-compute-phi} and the write of $\wpt$ on line \ref{ssca-write-phi} are
replaced by the read and write procedures of the CM sketch, respectively. Note that the input data $w$
on line \ref{ssca-first} is not of type int anymore but rather of type string. Consequently, the data needs 
to be hashed before the main iteration starts on line \ref{ssca-iterate-count}. We can replace the $V$ used
on line \ref{ssca-compute-phi} with the size of the hash space.

\paragraph{Implementation and Hardware}
\label{sec:apx-lda-hardware}
Our implementation is in Java. To achieve good performance, we use
only arrays of primitive types and preallocate all of the necessary
structures before the learning starts. We implement multi-threaded
parallelization within a node using the work-stealing Fork/Join
framework, and distribute across multiple nodes using the Java binding
to OpenMPI. Our implementation makes use of the Alias method to sample
the topics and leverage the document sparsity. We run our experiments
on a small cluster of 8 nodes connected through 10Gb/s Ethernet. Each
node has two 8-core Intel Xeon E5 processors (some nodes have Ivy
Bridge processors while others have Sandy Bridge processors) for a
total of 32 hardware threads per node and 512GB of memory. We use 4
MPI processes per node and set Java's memory heap size to 10GB.

\begin{figure*}
  \begin{itemize}
  \item[] $M$ \hfill (Number of documents)
  \item[] $\forall m \in \{1..M\}, N_m$ \hfill  (Length of document $m$)
  \item[] $V$ \hfill (Size of the vocabulary)
  \item[] $K$ \hfill (Number of topics)
  \item[] $\bm{\alpha} \in \mathbb{R}^K$ \hfill (Hyperparameter
    controlling documents)
  \item[] $\bm{\beta} \in \mathbb{R}^V$ \hfill (Hyperparameter
    controlling topics)
  \item[] $\forall m \in \{1..M\}, \theta_m \sim Dir(\bm{\alpha})$ \hfill (Distribution of topics in document $m$)
  \item[] $\forall k \in \{1..K\}, \phi_k \sim Dir(\bm{\beta})$ \hfill (Distribution of words in topic $k$)
  \item[] $\forall m \in \{1..M\}, \forall n \in \{1..N_m\}, z_{mn} \sim Cat(\theta_m)$ \hfill (Topic assignment)
  \item[] $\forall m \in \{1..M\}, \forall n \in \{1..N_m\}, w_{mn}
    \sim Cat(\phi_{z_{mn}})$ \hfill (Corpus content)
  \end{itemize}
\end{figure*}

\begin{figure*}
\begin{myalgorithmic}
\Procedure{$\mathit{SCA}$}{int $I$, int $M$, int $K$, int $N[M]$, float $\alpha$, float
  $\beta$, int $w[M][N]$}\label{ssca-first}
  \LocalArray{int $\mathit{tpd}[2][M][K]$}\label{ssca-decl-start}
  \LocalArray{int $\mathit{wpt}[2][H][R_2][K]$}
  \LocalArray{int $\mathit{wt}[2][K]$}\label{ssca-decl-end}
  \InitializeArray{$\mathit{tpd}[0]$}
  \InitializeArray{$\mathit{wpt}[0]$}
  \InitializeArray{$\mathit{wt}[0]$}
  \Remark{The main iteration}
  \For{$i$ from 0 through $(I \div 2)-1$}\label{ssca-iterate-start}\label{ssca-iterate-count}
    \For{$r$ from $0$ through $1$}\label{ssca-iterate-twice}
      \ClearArray{$\mathit{tpd}[1-r]$}
      \ClearArray{$\mathit{wpt}[1-r]$}
      \ClearArray{$\mathit{wt}[1-r]$}
      \Remark{Compute new statistics by sampling distributions}
      \Remark{\hskip0.6em that are computed from old statistics}
      \ForInRange{$m$}{$M$}
        \ForInRange{$n$}{$N$}
          \LocalArray{float $\mathit{p}[K]$}
          \Bind{$v$}{$w[m][n]$}
          \ForInRange{$k$}{$K$}
            \Bind{$\theta$}{$(\mathit{tpd}[r][m][k] + \alpha) / (N[m] + K \times \alpha)$}\label{ssca-compute-theta}
            \Bind{$\phi$}{$(\mathit{wpt}[r][v][k] + \beta) / (\mathit{wt}[r][k] + V \times \beta)$}\label{ssca-compute-phi}
            \Assign{$p[k]$}{$\theta \times \phi$}
          \EndFor
          \Bind{$z$}{$\mathit{sample}(p)$}   \Comment{\setbox0\hbox{\hskip0.6em (which may be integer}\hbox to \wd0{Now $0 \leq z < K$\hfil}}
          \Increment{$tpd[1-r][m][z]$}  \Comment{\setbox0\hbox{\hskip0.6em (which may be integer}\hbox to \wd0{Increment counters\hfil}}
          \Increment{$wpt[1-r][v][z]$}\label{ssca-write-phi} 
          \Increment{$wt[1-r][z]$} 
        \EndFor
      \EndFor
    \EndFor
  \EndFor\label{ssca-iterate-end}
\EndProcedure
\end{myalgorithmic}
\caption{Pseudocode for Streaming SCA}\label{fig:super-ssca}
\end{figure*}

\section{Microbenchmarks}
\label{sec:micro}

 In these microbenchmarks we measure mean relative error when
 estimating bigram frequencies using various CM sketches. The test set
 is a corpus of 2 million tokens drawn from a snapshot of Wikipedia
 with 769,722 unique bigrams.

\figref{fig:micro1} shows a comparison between CM with and without approximate counters.
 We use $w = 6666$ for the
 approximate counter sketch and $w = 3333$ for the conventional sketch,
 with $l = 3$ for both, so that the approximate counter sketches have
 twice as many total counters.  Since the most frequent bigram in this
 sample occurs 19430 times, the conventional sketch requires at least
 16 bit integers, while only 8 bits would suffice for each approximate
 counter, hence total space usage is approximately the same. We
 evaluate both kinds of sketches with and without conservative
 update. As expected, conservative update dramatically lowers error in
 both cases, particularly for less frequent words. The larger value of
 $w$ enabled by approximate counters further lowers error on less
 frequent words, though for more frequent words there is some increased
 error.
 Additional benchmarks measuring the effect of the counter
 base and errors resulting from merging sketches are given in Appendix~\ref{sec:apx-experiments}.
 
 \begin{figure*}
 \begin{center}
   \subfigure[ \label{fig:micro1a}]{
     \includegraphics[width=0.45\textwidth]{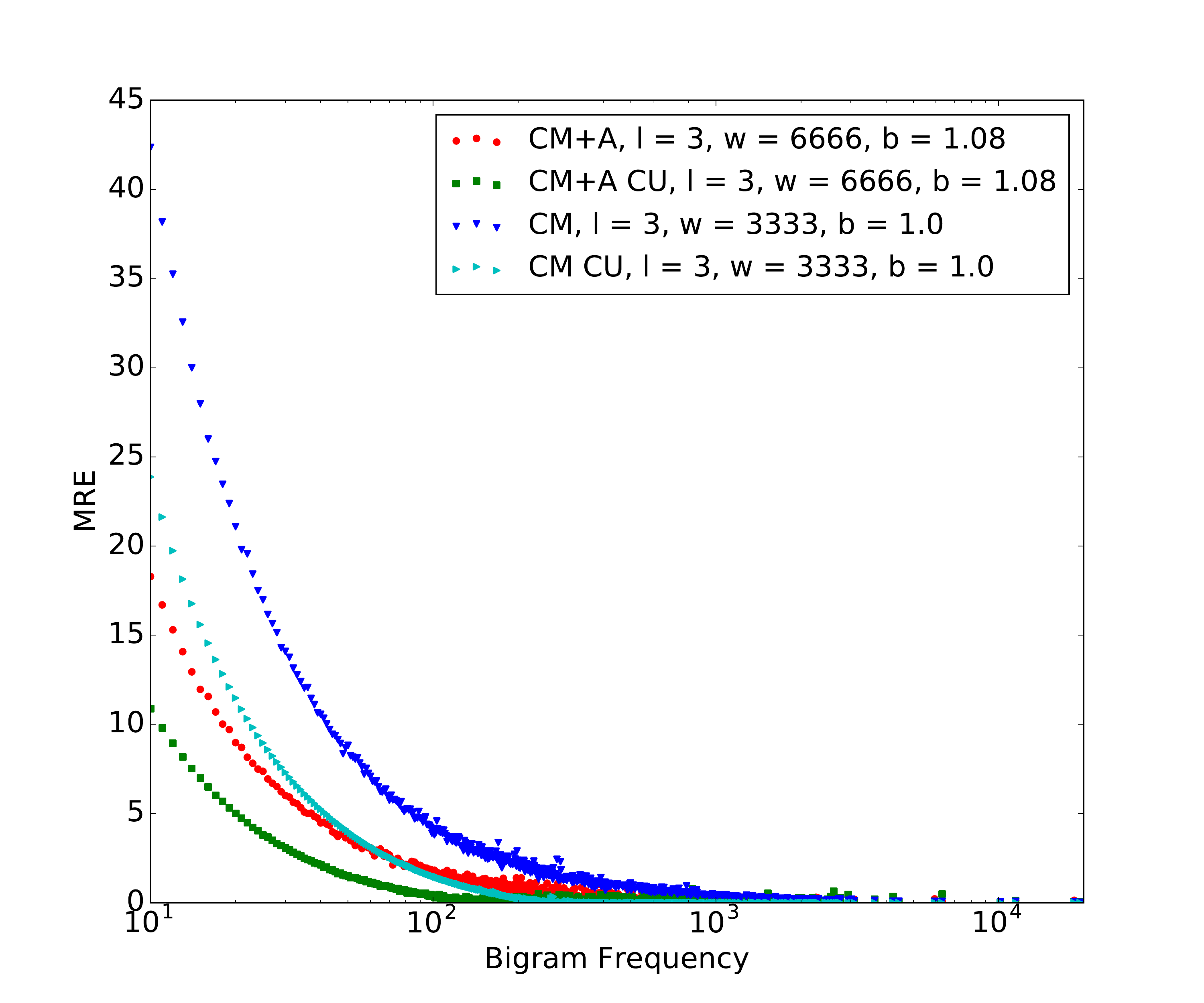}
    }
   \subfigure[ \label{fig:micro1b}]{
     \includegraphics[width=0.45\textwidth]{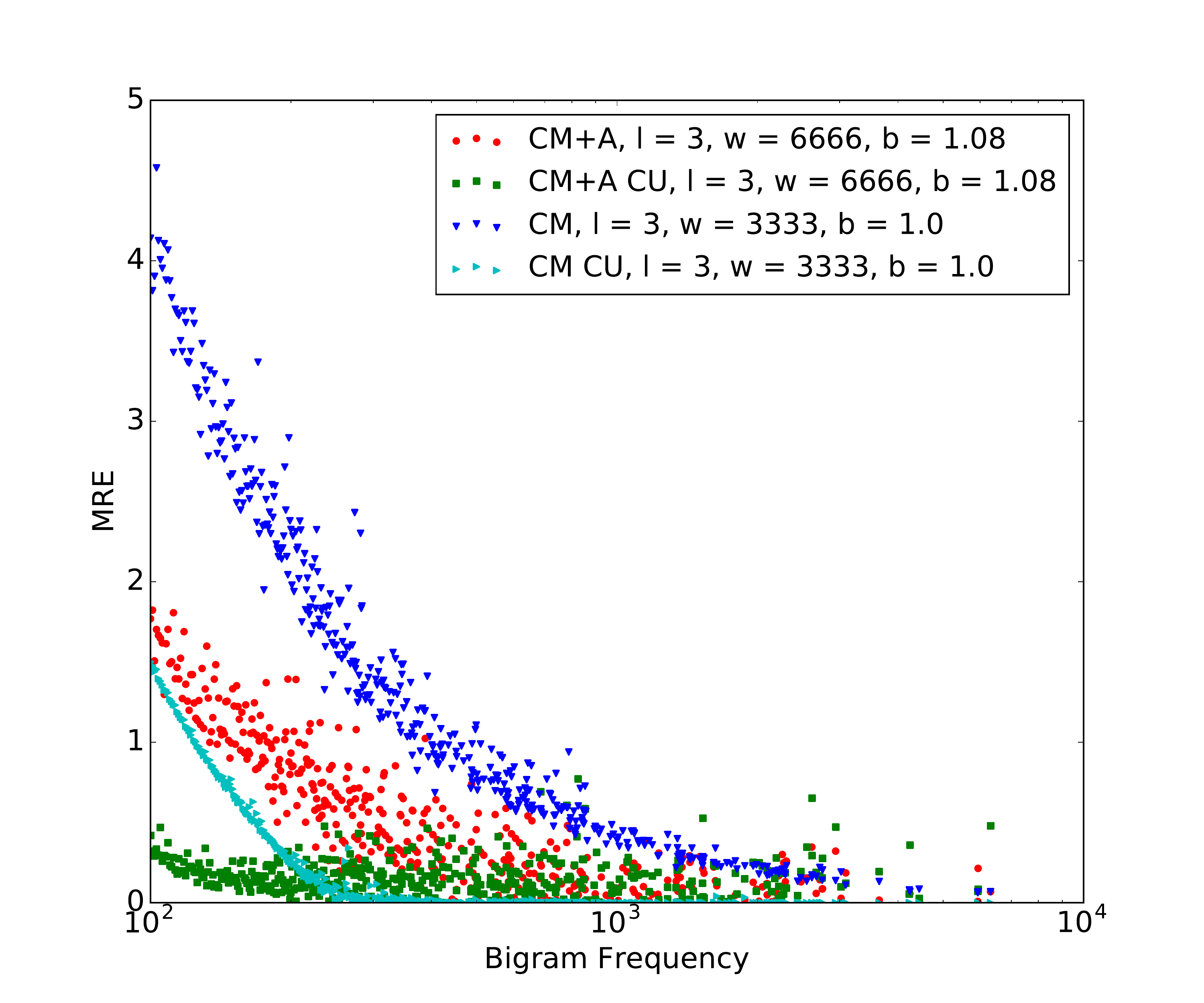}
   }
   \caption{Experiment measuring mean relative error for bigram estimation. The $y$-value of each point on the plot is equal to the mean relative error for all bigrams whose true frequency is equal to the $x$-value of the point, averaged over $5$ trials; the right figure is the same data zoomed to a subset of the $x$-axis. ``CM+A'' is a combined sketch using independent counters; ``CM+A CU'' is a combined sketch with correlated conservative update. ``CM CU'' and ``CM'' are conventional sketches with/without conservative update, respectively.}
   \label{fig:micro1}
 \end{center}
\end{figure*}

\figref{fig:base_normal} and \figref{fig:base_cu} show the effects of varying
the counter base $b$ when using the independent counter and conservative update
alternatives. There are two interesting phenomena in these plots. First, when
using the independent counters, we see that the error for low frequency words is
slightly but consistently lower when the base is larger. This makes sense
because we know that the sketch overestimates such words, but when we take the
minimum of independent counters the expected value of the result is smaller,
which reduces the error for infrequent words -- with a larger counter base, it
is more likely that the minimum counter happens to underestimate. The
second effect is that when using conservative update with a large base, there
appear to be more instances with large error when estimating frequent words,
compared to the independent case. Again, a plausible explanation is that these
errors correspond to cases where the approximate counters occasionally take on a
much larger value than the true count, and that when taking the minimum of
several independent counters this is unlikely to happen.

\begin{figure*}
\begin{center}
  \subfigure[ \label{fig:base_normal1}]{
    \includegraphics[width=0.45\textwidth]{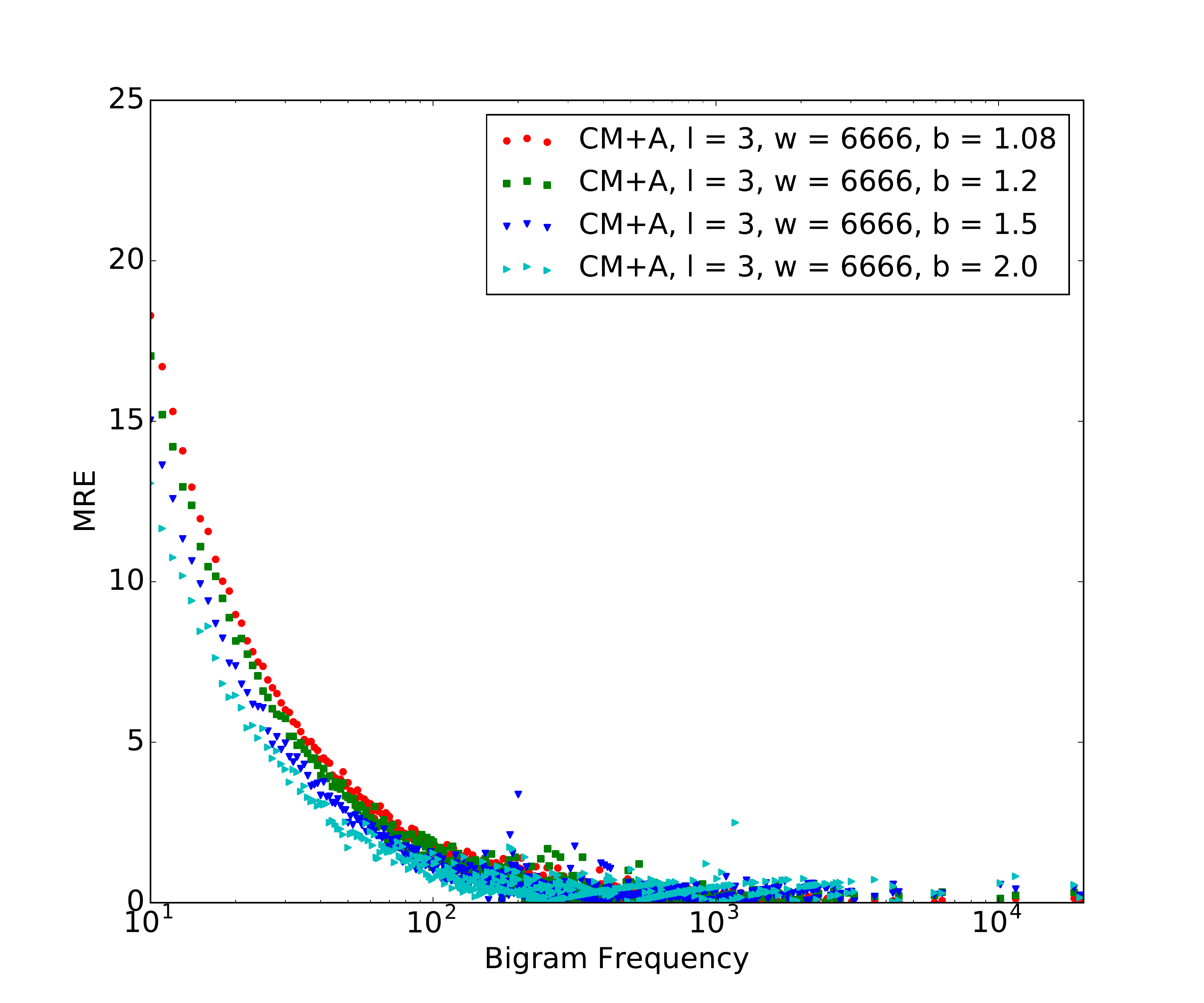}
   }
  \subfigure[ \label{fig:base_normal2}]{
    \includegraphics[width=0.45\textwidth]{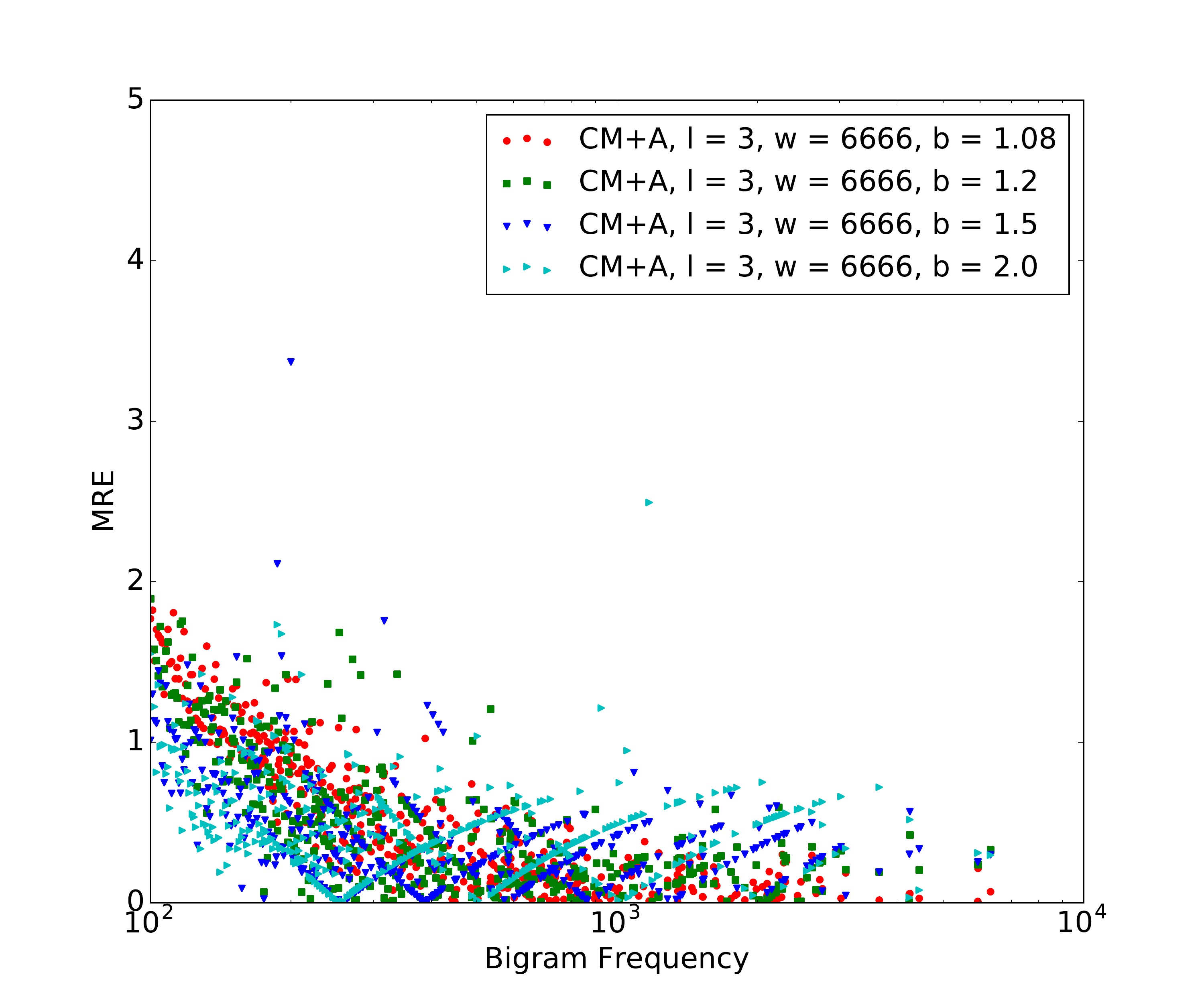}
  }
  \caption{Effects of different approximate counter base $b$ for bigram estimation benchmark using independent counters. The figure on the right is a zoomed in version of the same data.}
  \label{fig:base_normal}
\end{center}
\end{figure*}

\begin{figure*}
\begin{center}
  \subfigure[ \label{fig:base_cu1}]{
    \includegraphics[width=0.45\textwidth]{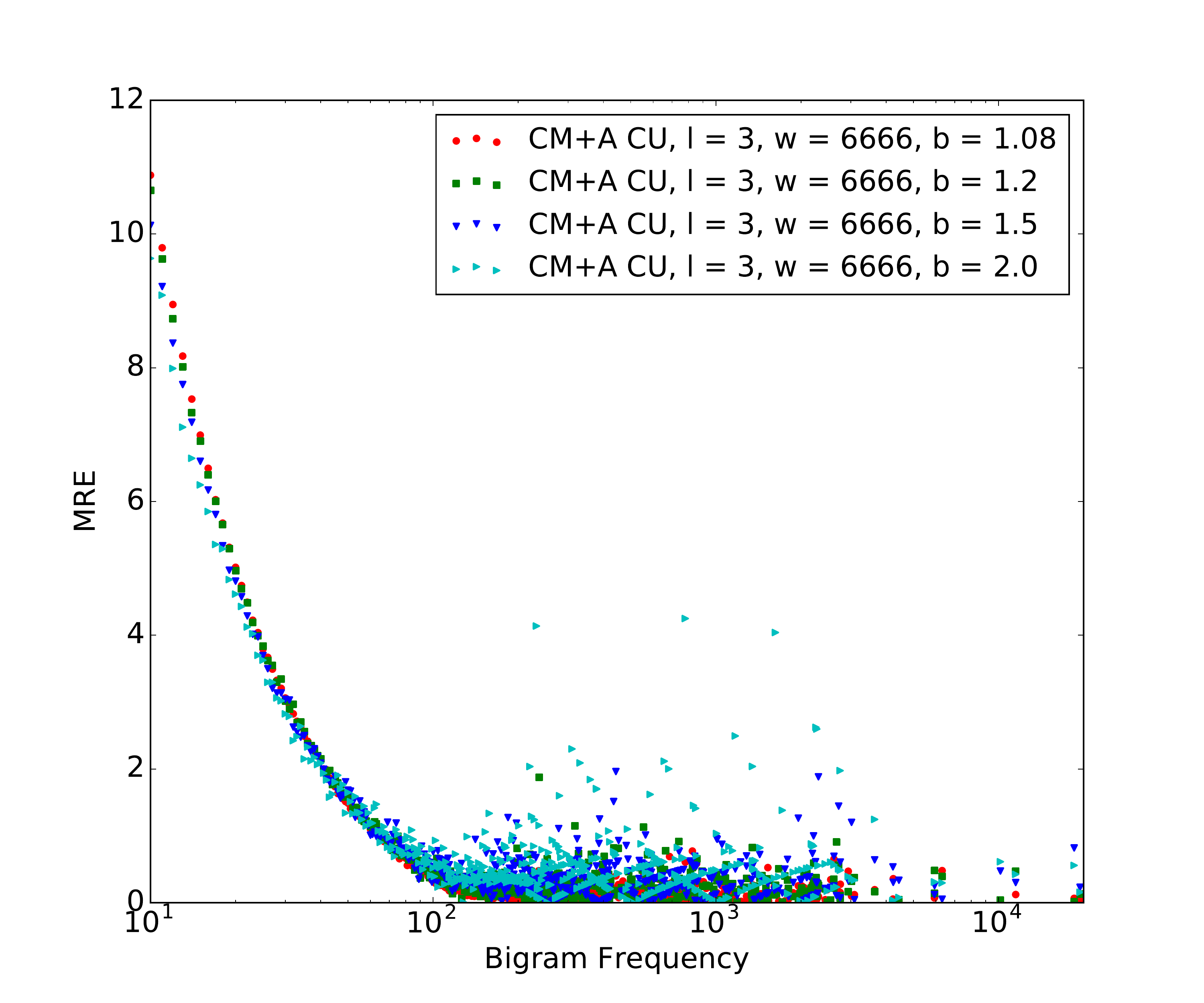}
   }
  \subfigure[ \label{fig:base_cu2}]{
    \includegraphics[width=0.45\textwidth]{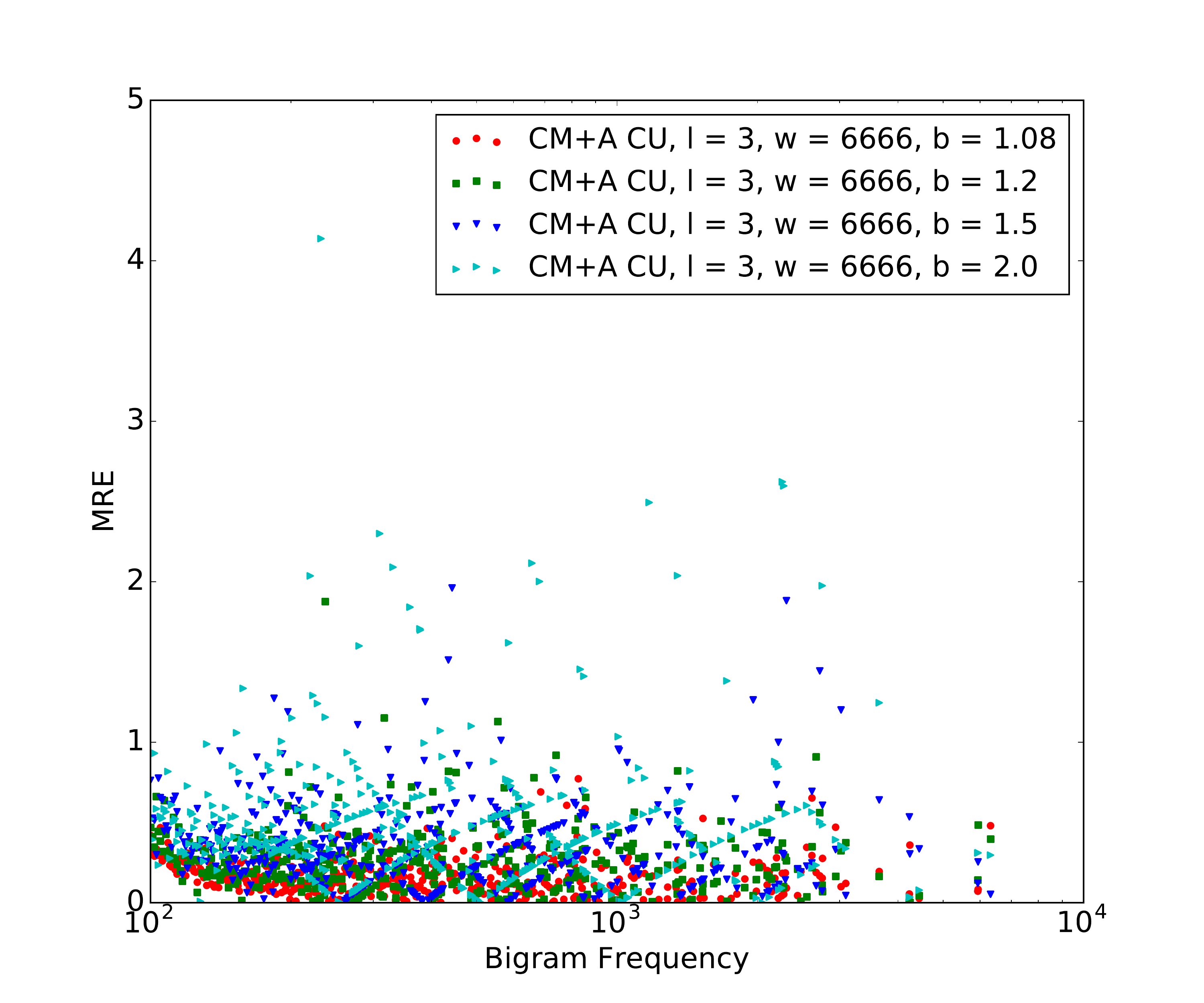}
  }
  \caption{Effects of different approximate counter base $b$ for bigram estimation benchmark using conservative update.}
  \label{fig:base_cu}
\end{center}
\end{figure*}

\figref{fig:streak} shows just the points for the $b=2$ conservative update case from these same experiments. There is a noticeable repeating ``crisscross'' pattern in the errors for the more frequent keys. This arises because the base is so large the counter can only represent counts of the form $2^k-1$, so the mean relative error for these keys is largely dependent on how close they are to a value that can be represented in this form.

\begin{figure*}
  \begin{center}
    \includegraphics[width=0.75\textwidth]{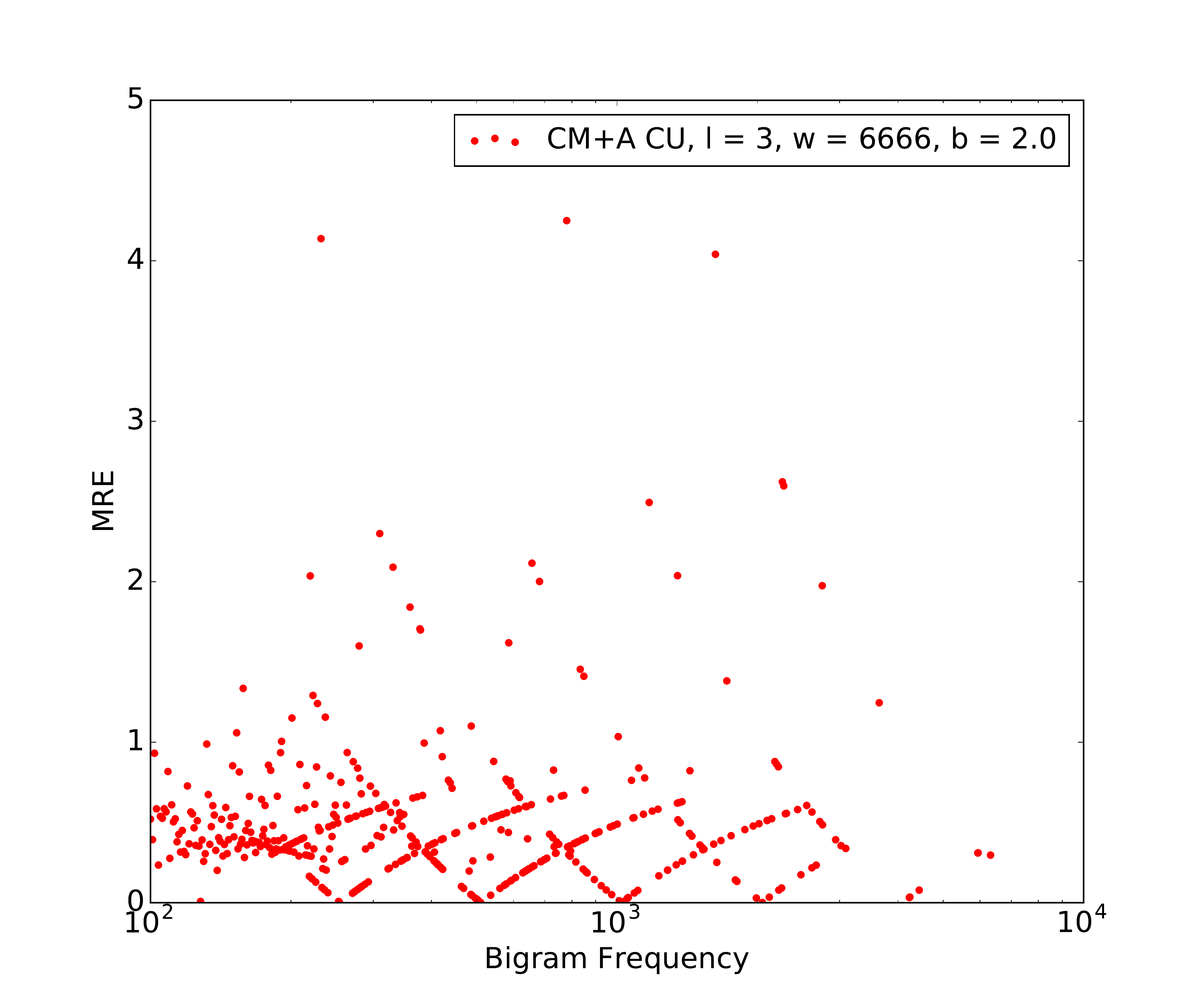}
  \end{center}
  \caption{Results showing crisscross pattern for large frequency key errors with large bases.}
  \label{fig:streak}
\end{figure*}

Finally, \figref{fig:splits} shows the error when merging
different sketches together. In these experiments, we split the corpus
into $k$ substreams of equal size, compute sketches on each, and then
merge the $k$ sketches together. We vary $k$ from $1$ to $4$, and use conservative update
when computing each of the subsketches. There does not appear 
to be a substantial difference between the errors when merging or not when using $b = 1.08$.

\begin{figure*}
\begin{center}
  \subfigure[ \label{fig:splits1}]{
    \includegraphics[width=0.45\textwidth]{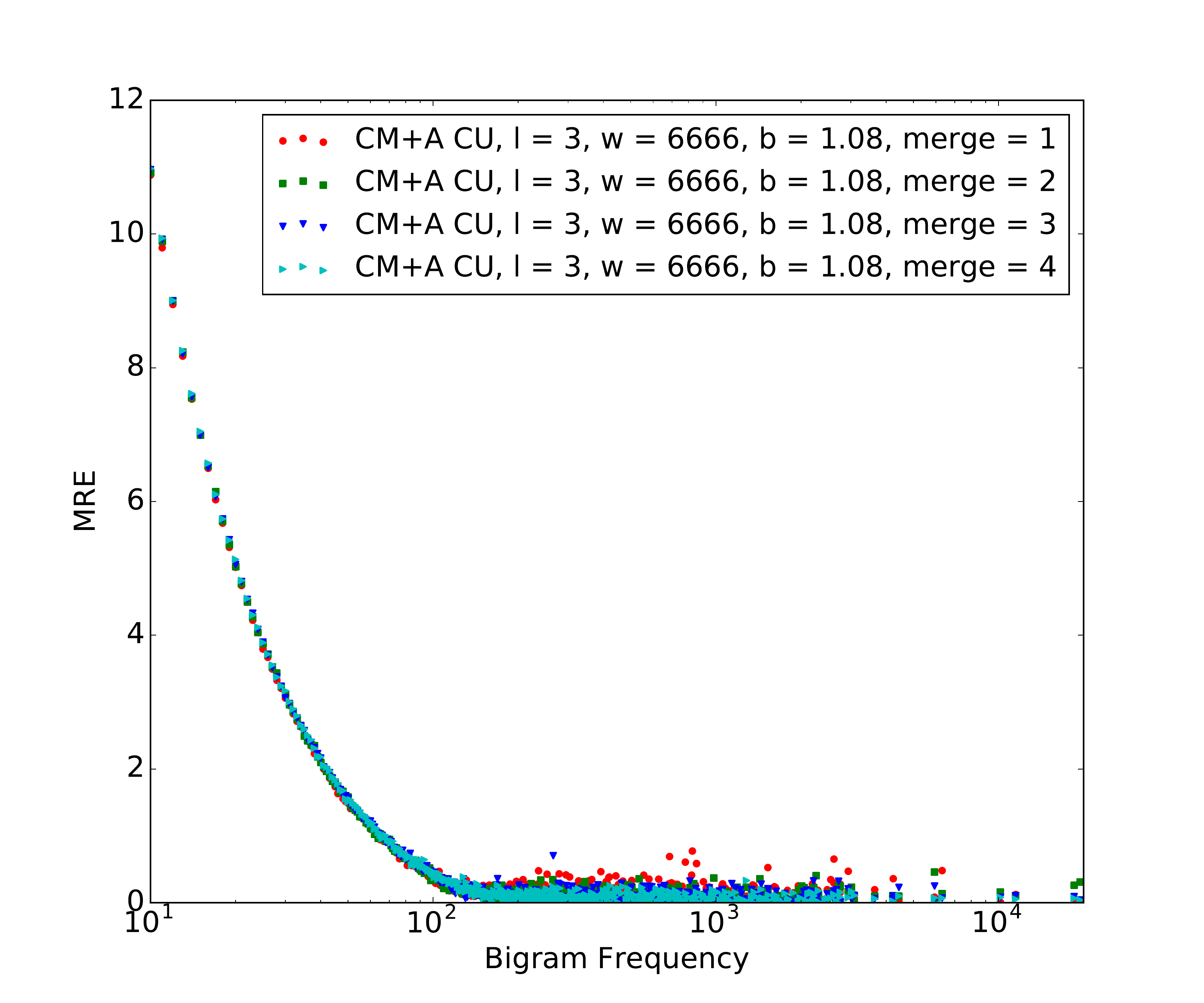}
   }
  \subfigure[ \label{fig:splits2}]{
    \includegraphics[width=0.45\textwidth]{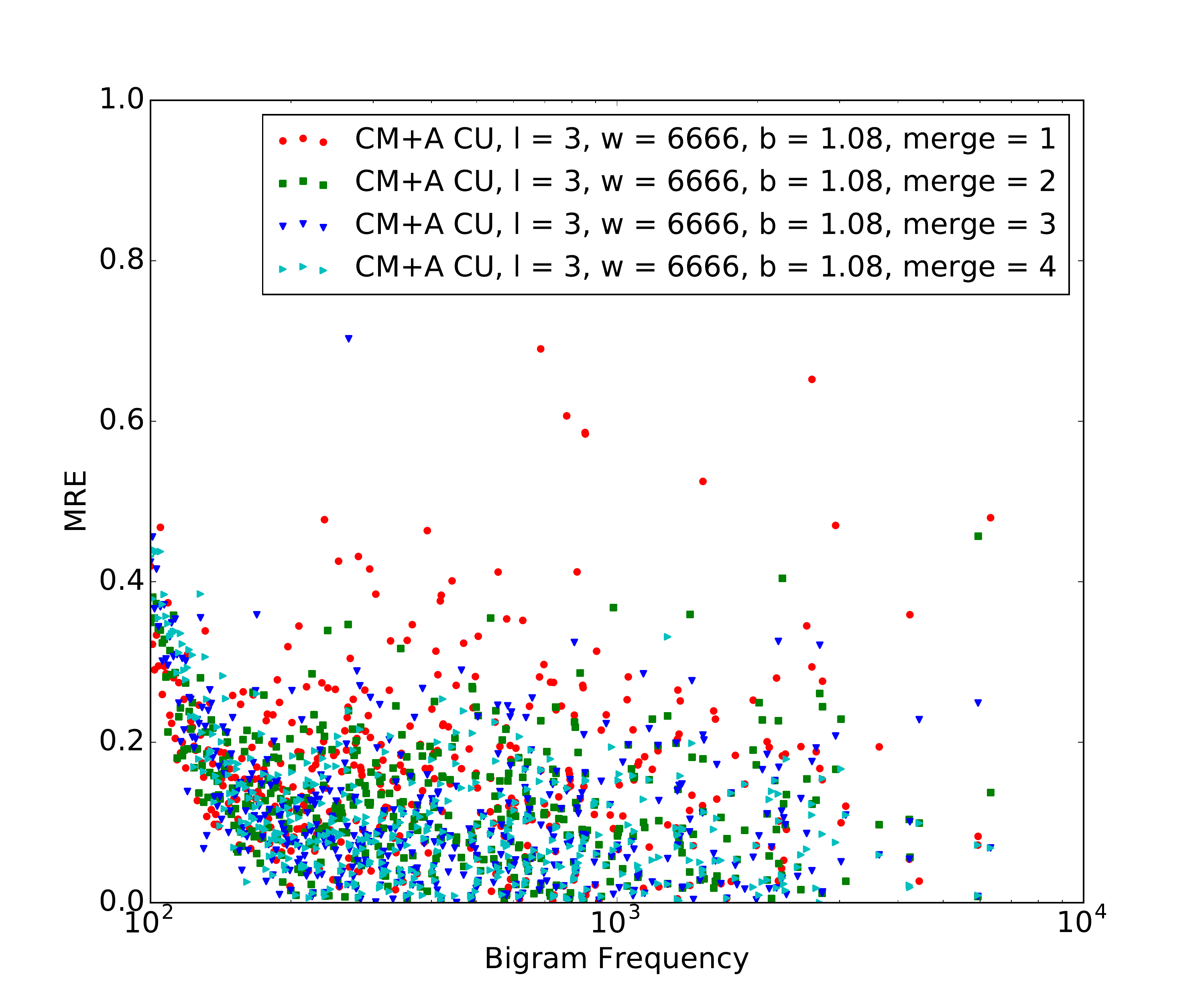}
  }
  \caption{Effects of merging multiple sketches together using the addition procedure of \citet{Adding}.}
  \label{fig:splits}
\end{center}
\end{figure*}

\section{Experiments}
\label{sec:apx-experiments}

In this section we present additional experiments, provide example
topics and report timing results for the various algorithms.

\subsection{Example Topics}
\label{sec:apx-topics}

Since it is possible for a model to have good perplexity, but yield
poor topics, we also manually inspect the top-words per topics to
ensure that the topics are reasonable in our LDA experiments.  In
general, we find no perceptible difference in quality between the
models that employ sketching and hashing representations of the
sufficient statistics and those that employ exact array-based
representations.  In Table~\ref{tab:topics} we provide example topics
from three different models, all of which use representations that
compress the sufficient statistics more than a traditional array of
4-byte integers.  The first two systems combine count-min sketch with
approximate counters while the third combines all three ideas:
count-min sketch, feature hashing and approximate counters.  As is
typical of LDA and other topic-models, not all topics are perfect and
some topics arise out of idiosyncracies of the data like a topic full
of dates from a certain century, but again, there did not appear to be
noticeable differences in quality between the systems.

\cut{
\begin{table}[t]
\begin{centering}
\begin{tabular}{|c|c|c|}
\hline
{\bf Topic A} & {\bf Topic B} & {\bf Topic C} \\
\hline
canada     & party & king \\
canadian  & election & family \\
hockey & elections & son \\
ontario & elected & father \\
toronto & democratic & married \\
ice & political & died \\
columbia & state & death \\
quebec & assembly & daughter \\
montreal & republican & brother \\
nhl & house & wife \\
\hline
\end{tabular}
\caption{Example topics and their top-10 most likely words for the
  final streaming system with CM sketch (3 18-bit hash functions),
  feature-hashing (23 bit hash), and approximate counters (8 bits, 1.08 base).}
\label{tab:topics}
\end{centering}
\end{table}
}

\cut{
\begin{table}
\centering
  \begin{tabular}[b]{|c|c|c|}
    \hline
    {\bf Topic A} & {\bf Topic B} & {\bf Topic C} \\
    \hline
    band     & law  & road \\
    album  & court & route \\
    released & act & bridge \\
    records & police & highway \\
    rock & case & north\\
    music & states & street \\
    live & united & east \\
    track & legal & west \\
    release & rights & south \\
    guitar & state & state \\
    \hline
  \end{tabular}
\caption{Example topics and their top-10 most likely words for the
  final streaming system with count-min sketch (3 18-bit hash functions),
  feature-hashing (22 bit hash), and approximate counters (8 bits, 1.08 base).}
\label{tab:topics}
\end{table}
}

\begin{table}
\centering
  \begin{tabular}[b]{|c|c|c|}
    \hline
    {\bf Topic A} & {\bf Topic B} & {\bf Topic C} \\
    \hline
    \multicolumn{3}{|c|}{CM+A 3 16 1.08} \\
    \hline
    space & episode & law \\
    light & series & court \\
    energy & season & act \\
    system & show & states \\
    earth & episodes & legal \\
    \hline
    \multicolumn{3}{|c|}{CM+A 3 15 1.04} \\
    \hline
    russian & john & man \\
    war & william & time \\
    government & died & story \\
    soviet & king & series \\
    union & henry & back \\
    \hline
    \multicolumn{3}{|c|}{CM+H+A 3 18 22 1.08} \\
    \hline
    system & band & court \\
    high & album & police \\
    power & released & case \\
    systems & rock & law \\
    device & records & prison \\
    \hline
  \end{tabular}
  \caption{Example topics. ``CM+A 3 16 1.08'' indicates CM-sketch
    (three 16-bit hashes) with approximate counters (8-bit base-1.08),
    ``CM+A 3 15 1.04'' indicates CM-sketch (three 15-bit hashes) with
    approximate counters (8-bit base-1.04), and ``CM+H+A 3 16 22
    1.08'' indicates CM-sketch (three 18-bit hashes) with feature hashing
    (22-bit) and approximate counters (8-bit base-1.08).  All three
    configurations are more compressive than the traditional array-based
    representation of the sufficient statistics that utilize 4-byte integers.}
\label{tab:topics}
\end{table}

\subsection{Timing results}
\label{sec:apx-timing}

\begin{table}
\centering
\begin{tabular}{|l|r|r|r|}
\hline
{\bf Method} &  {\bf \# hashes} & {\bf hash range} & {\bf time (s)} \\
\hline
SSCA & NA & NA & 12.14 $\pm$ 1.82\\
\hline
SSCAS & 3 & 15 & 22.75  $\pm$ 4.30\\ 
SSCAS & 3 & 16 & 23.90  $\pm$ 4.41\\
SSCAS & 3 & 17 & 25.32  $\pm$ 4.68\\
SSCAS & 3 & 18 & 28.10  $\pm$ 5.18\\
\hline
SSCAS & 4 & 15 & 29.70 $\pm$  5.82\\
SSCAS & 4 & 16 & 32.75  $\pm$ 6.17\\
SSCAS & 4 & 17 & 33.35  $\pm$ 5.89\\
SSCAS & 4 & 18 & 36.18  $\pm$ 5.97\\
\hline
SSCAS & 5 & 15 & 37.76  $\pm$ 6.95\\
SSCAS & 5 & 16 & 39.71  $\pm$ 7.01\\
SSCAS & 5 & 17 & 42.33  $\pm$ 7.75\\
SSCAS & 5 & 18 & 45.47 $\pm$  7.70\\
\hline
SSCAS & 6 & 15 & 45.04 $\pm$  8.51\\
SSCAS & 6 & 16 & 46.38 $\pm$  8.18\\
SSCAS & 6 & 17 & 48.70 $\pm$  8.21\\
SSCAS & 6 & 18 & 53.03 $\pm$  8.89\\
\hline
SSCAS & 7 & 15 & 51.66  $\pm$ 9.45\\
SSCAS & 7 & 16 & 53.35 $\pm$  9.38\\
SSCAS & 7 & 17 & 56.44 $\pm$  9.42\\
SSCAS & 7 & 18 & 61.15 $\pm$  9.93\\
\hline
\end{tabular}
\caption{Time per iteration results for LDA with CM sketch.}
\label{tab:time-cmsketch}
\end{table}

We report the timing results in this section.

\paragraph{CM sketch timing}

Although the CM sketch representation of the sufficient
statistics behaves well statistically, there are some computational
concerns.  In particular, the CM sketch stores multiple counts
per word (one per-hash function) and a distributed algorithm must then
communicate the extra counts over the network.  Therefore, to evaluate
the effect on run-time performance, we also report the average
per-iteration wall-clock time for each system in
Table~\ref{tab:time-cmsketch}. The method ``SSCA'' is the default
implementation of SCA the employs arrays to represent the sufficient
statistics and methods marked ``SSCAS'' employ the CM sketch.

As expected, the number of hash functions has a bigger impact on
running time than the range of the hash functions.  Fortunately, at
least empirically, the range of the hash function is more important
than the number of hash functions in that it has a greater effect on
inference's ability to achieve higher accuracy in fewer iterations.
Thus, in terms of both the number of iterations and the wall-clock
time, increasing the hash range is more beneficial than increasing the
number of hashes.

\paragraph{Timing with approximate counters}

Although using the CM sketch with these dimensions increases the running time as demonstrated
above, approximate counters effectively compensate for the increased
communication overhead because the corresponding sketches are much smaller for a given number of hash functions and range.  We report the results in
Table~\ref{tab:time-overcome}.  The first row, method ``SSCA'' is the
default implementation of SSCA using array representation of the
sufficient statistics and ``SSCASA'' is the version with sketching and
approximate counters (8 bits and base 1.08).

\begin{table}
\centering
\begin{tabular}{|l|r|r|r|}
\hline
{\bf Method} &  {\bf \# hashes} & {\bf hash range} & {\bf time (s)} \\
\hline
SSCA & NA & NA & 12.14 $\pm$ 1.82\\
\hline
SSCASA & 3 & 15 & 12.58 $\pm$ 2.00\\
SSCASA & 3 & 16 & 17.57 $\pm$ 2.78\\
SSCASA & 3 & 17  & 22.69 $\pm$ 3.72\\
SSCASA & 3 & 18 & 24.29 $\pm$ 3.93\\
\hline
SSCASA & 4 & 15  & 17.00 $\pm$ 2.82\\
SSCASA & 4 & 16 & 24.50 $\pm$ 4.18\\
SSCASA & 4 & 17 & 29.46 $\pm$ 4.88\\
SSCASA & 4 & 18  & 32.19 $\pm$ 5.30\\
\hline
SSCASA & 5 & 15  & 22.11 $\pm$ 3.84\\
SSCASA & 5 & 16 & 31.26 $\pm$ 5.54\\
SSCASA & 5 & 17 & 37.78 $\pm$ 6.60\\
SSCASA & 5 & 18 & 41.12 $\pm$ 6.74\\
\hline
\end{tabular}
\caption{Time per iteration results for LDA with CM sketch and approximate counters.}
\label{tab:time-overcome}
\end{table}

\subsection{Exploring more sketching and hashing parameters}
\label{sec:apx-all-params}

In this section, we report a wider range of settings to the parameters
of the CM sketch.  In particular, we report numbers for a
sketch with a range of just 15-bits to one with 18, while also varying
the number of hash functions from 3 to 7.  To make the plots more
readable, we depict curves for sketches with the same hash range in
the same color (for example, all sketches with a 15-bit range are
red).  We also depict curves for sketches with the same number of hash
functions with the same symbol (for example, sketches with three hash
functions are all marked with circles).

In Figure~\ref{fig:cmsketch-all} we report the results for just the
CM sketch and begin to reach the limits of our ability to push
the compressiveness of the sketch.  We see that the final perplexity
(after 60 iterations) is the same in all cases except for the most
compressive variant of the sketch (with 15 bits and 3 hash functions).
Since the vocabulary size is 291,561, this sketch is one-third of the
size as the raw array-based representation for representing word
counts per topic.  This seems to be the point at which we begin to see
worse final perplexity.

We can also see from this plot that the hash range (curves from the
same hash-range are in the same color, and curves with different
ranges are in different colors) has a bigger impact on initial
performance than the number of hash functions.  For example, if we
wanted to double the size of the sketch, it would be better to add an
extra bit to the range than to double the number of hash functions
from 3 to six (as seen by the perplexity gaps in the early iterations
between each of the hash ranges). This is in line with the earlier results of \citet{GoyalD11, GoyalDC12}.

In Figure~\ref{fig:cmsketch-ac-all} we repeat the same experiment, but
employ an 8-bit base-1.08 approximate counter to represent the counts
in the sketch (instead of the usual 4-byte integers).  We include up
to 5 hashes for this plot.  Note that despite using just one-quarter
the amount of memory to represent the sufficient statistics, the
results for these combined data-structures are similar to the
CM sketch alone.  Further, as noted earlier, the approximate
counters are much faster in a distributed setting since they overcome
the additional data that needs to be transmitted by the CM
sketch.  Thus, the combined data-structure is not only more
compressive than the CM sketch alone, but it also runs much
faster, achieving similar performance as the original algorithm
depending on the setting to its parameters.

Finally, as we mentioned in Section~\ref{Analysis}, combining the
CM sketch and the approximate counters is non-obvious due to
the way the min operation interacts with the counters.  We had
proposed and discussed several alternatives: CM sketch with
independent counters, counter min-sketch with correlated counters, and
CM sketch with correlated counters and the conservative update
rule to reduce the bias. We show the plots for these counters in
Figures~\ref{fig:combined_i},~\ref{fig:combined_s},~\ref{fig:combined_f}
respectively.  We also vary the base of the counters while keeping the
number of counter bits fixed at 8.  Each color represents a different
base (1.08, 1.09 and 1.10) to make it easier to interpret.  The main
takeaways from these plots is that the method in which you combine the
counters and min-sketch does not matter as much for an application
like LDA, which seems to be robust to the bias in the first two
methods. We note that in some cases, increasing the base of the counter appears to improve perplexity.
While we have not been able to find a satisfactory explanation for this phenomenon, 
previous work on the geometric aspects of topic modeling~\cite{Geometry1,Geometry2} has
highlighted the subtle interaction between the geometry of the topic simplex and perplexity.

\begin{figure*}[tb]
\begin{center}
  \subfigure[CM sketch only \label{fig:cmsketch-all}]{
    \includegraphics[width=0.35\textwidth]{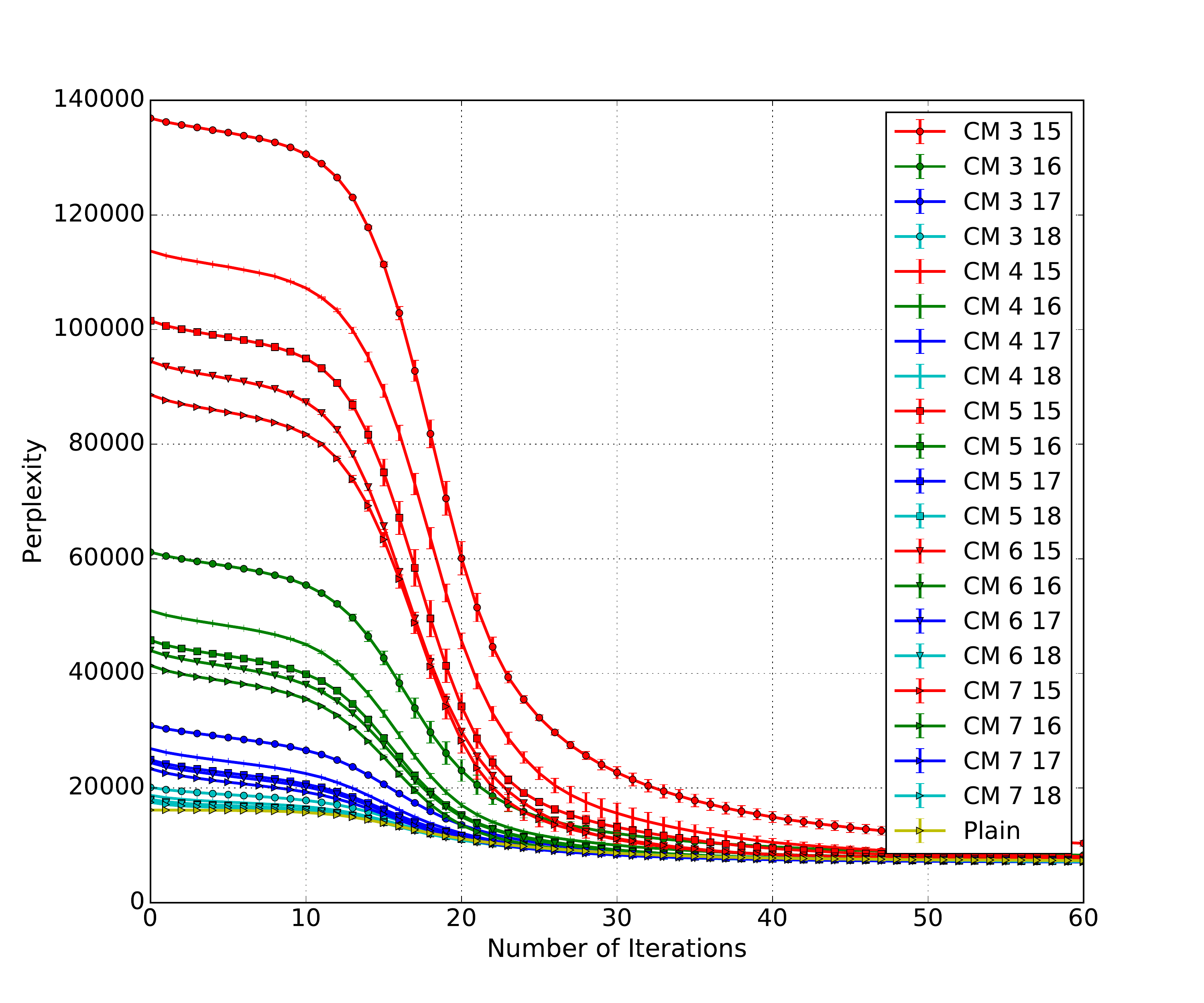}
   }
  \subfigure[sketch + approx. counters  \label{fig:cmsketch-ac-all}]{
    \includegraphics[width=0.35\textwidth]{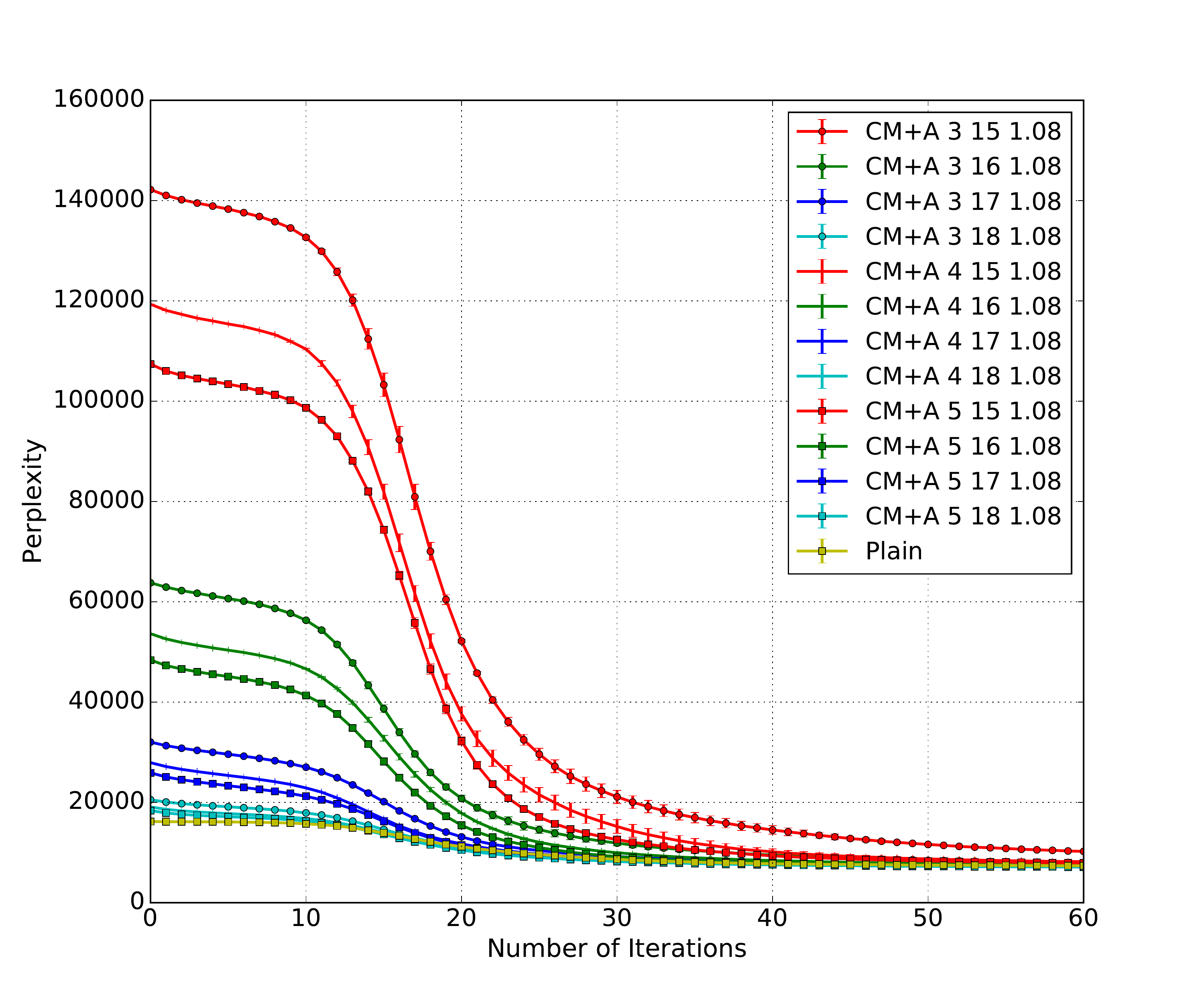}
  }
\caption{Experiments across a full range of parameter settings for the
sketch.}
\label{fig:all-sketch-params}
\end{center}
\end{figure*}

\begin{figure*}[tb]
\begin{center}
  \subfigure[15 bits \label{fig:combo_i_15}]{
    \includegraphics[width=0.22\textwidth]{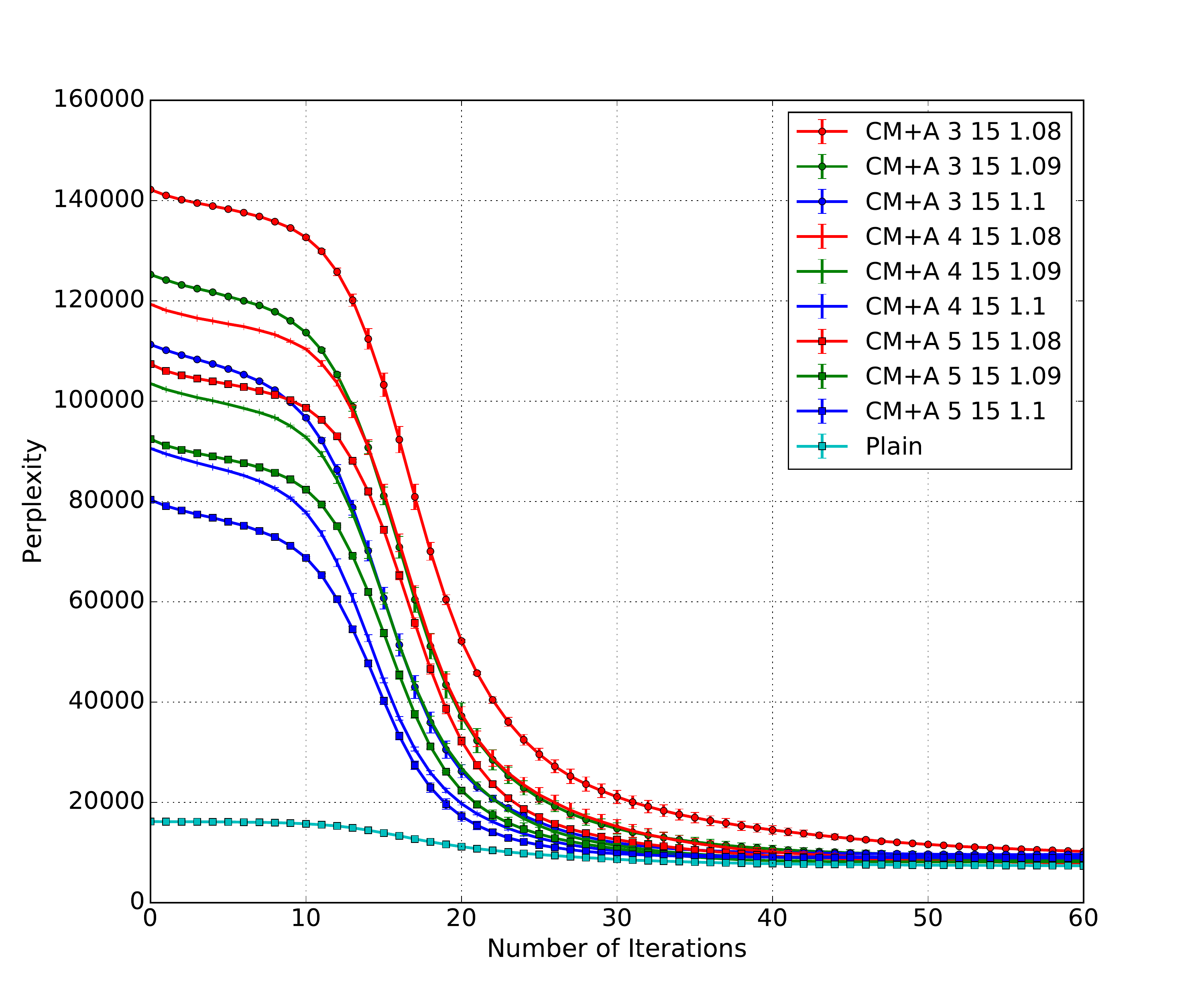}
   }
  \subfigure[16 bits \label{fig:combo_i_16}]{
    \includegraphics[width=0.22\textwidth]{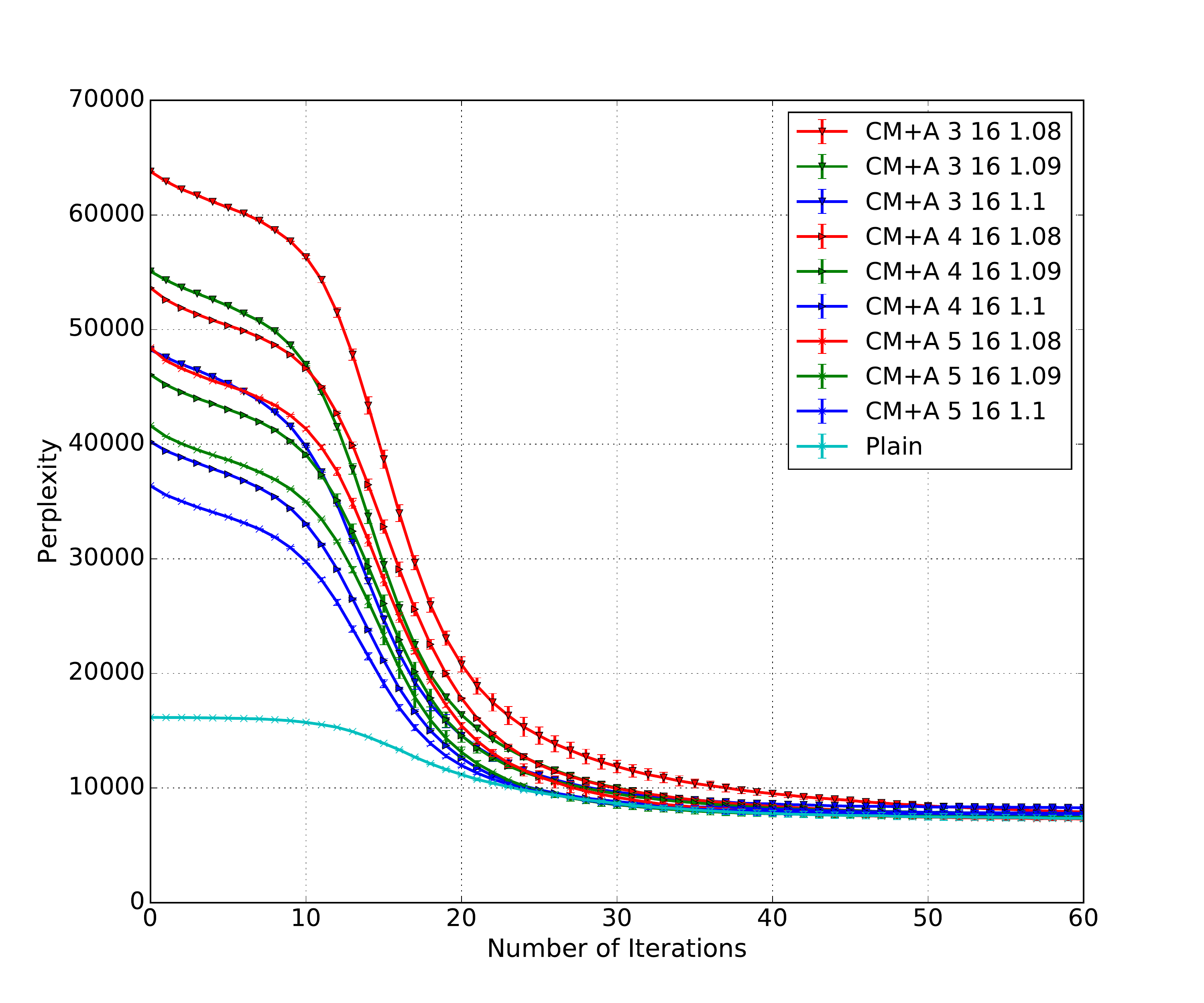}
  }
  \subfigure[17 bits \label{fig:combo_i_17}]{
    \includegraphics[width=0.22\textwidth]{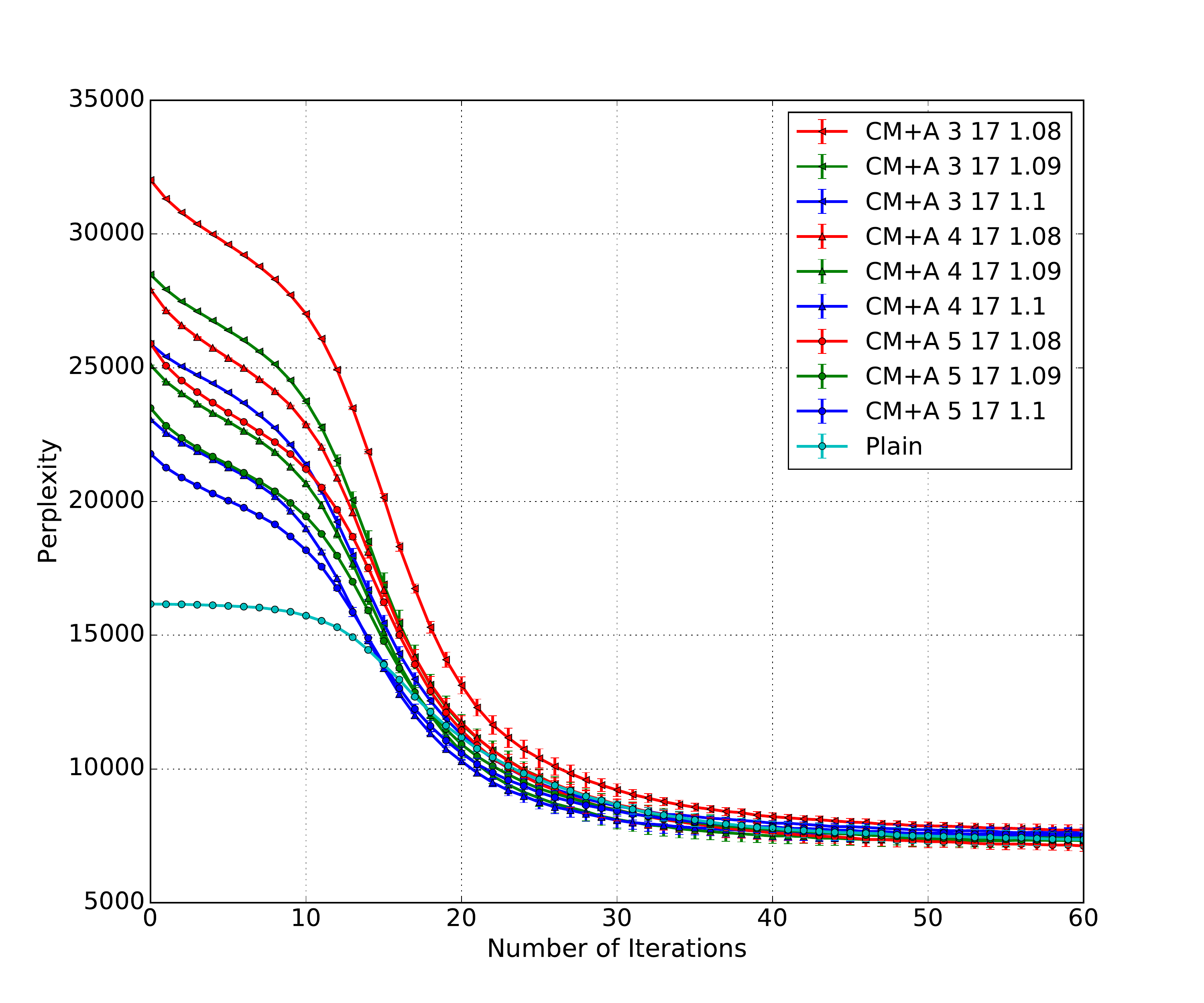}
  }
  \subfigure[18 bits  \label{fig:combo_i_18}]{
    \includegraphics[width=0.22\textwidth]{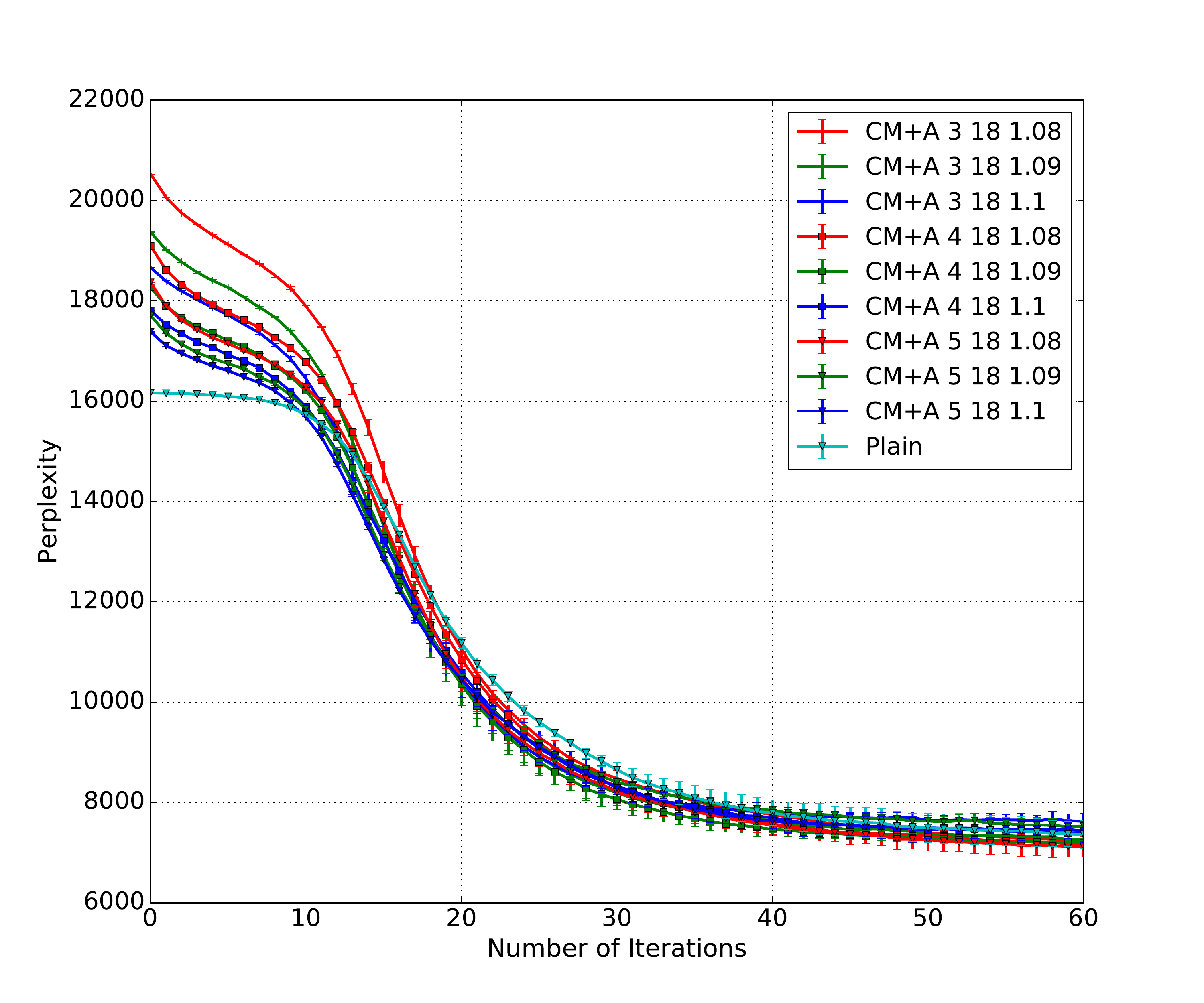}
  }
\caption{{\em Independent counter} variant of the combined CM sketch
  and approximate counter data-structure on LDA  with a hash range of 15,16,17 and 18 as
  indicated in the respective captions.}
\label{fig:combined_i}
\end{center}
\end{figure*}

\begin{figure*}[tb]
\begin{center}
  \subfigure[15 bits \label{fig:combo_s_15}]{
    \includegraphics[width=0.22\textwidth]{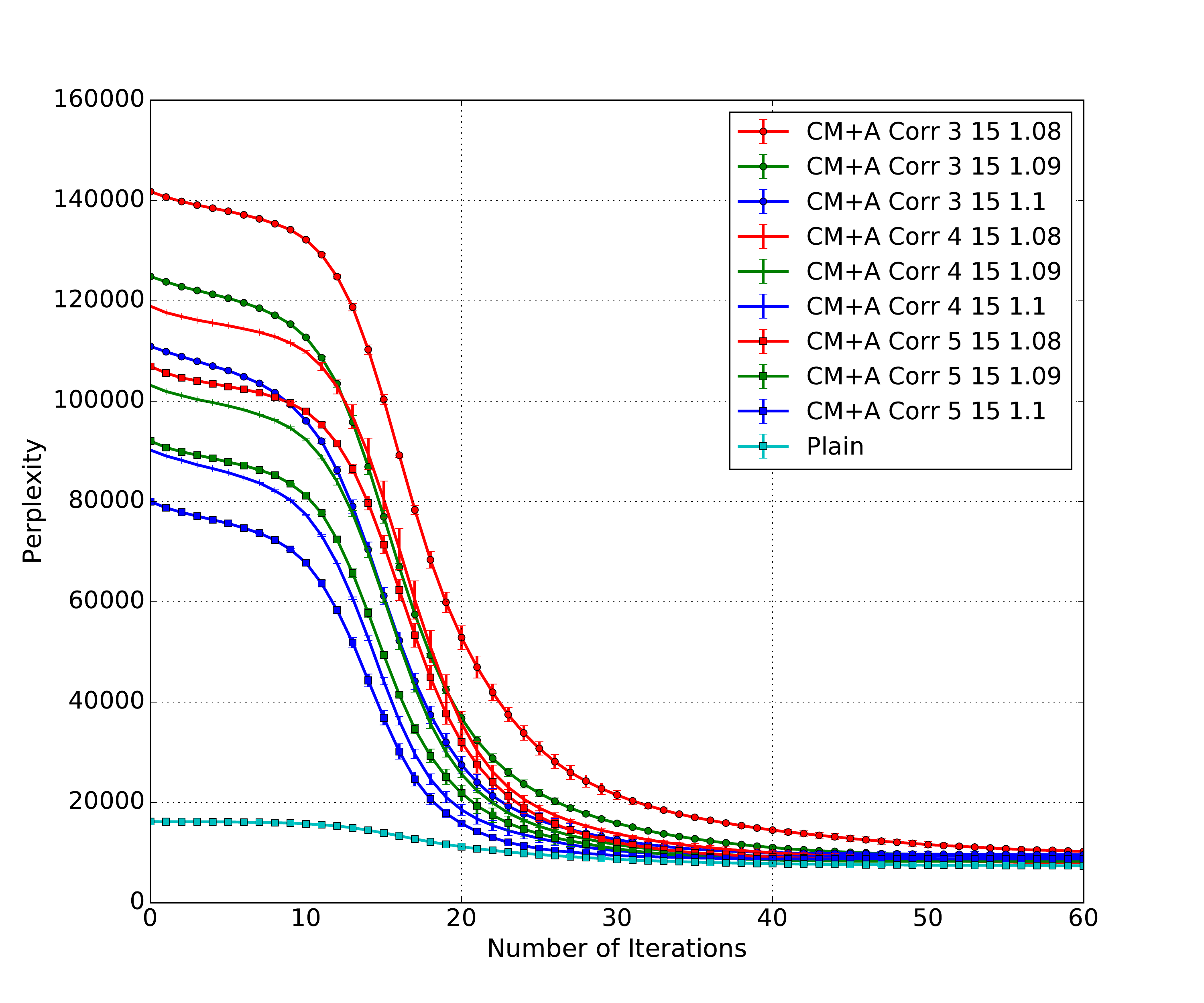}
   }
  \subfigure[16 bits \label{fig:combo_s_16}]{
    \includegraphics[width=0.22\textwidth]{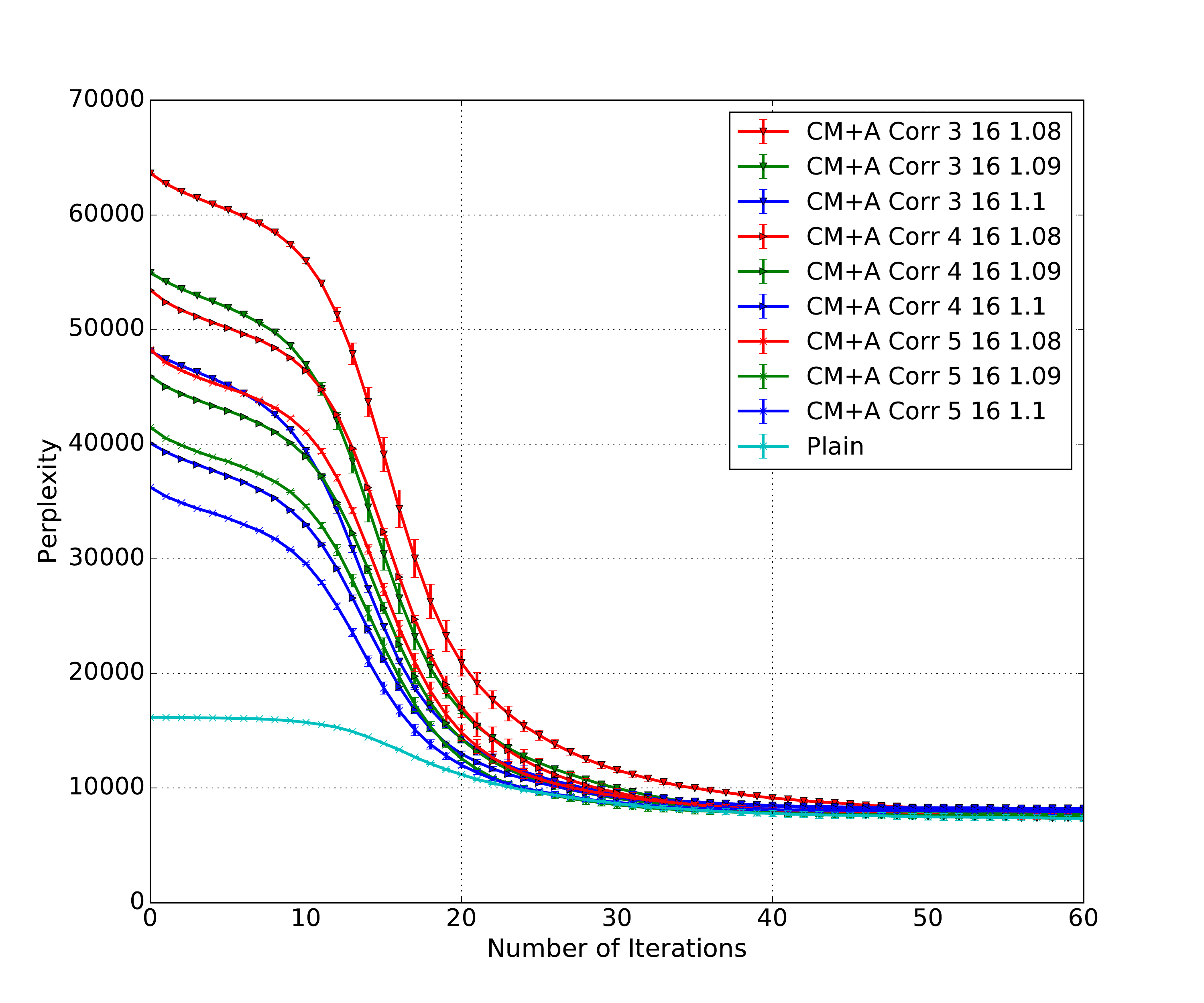}
  }
  \subfigure[17 bits \label{fig:combo_s_17}]{
    \includegraphics[width=0.22\textwidth]{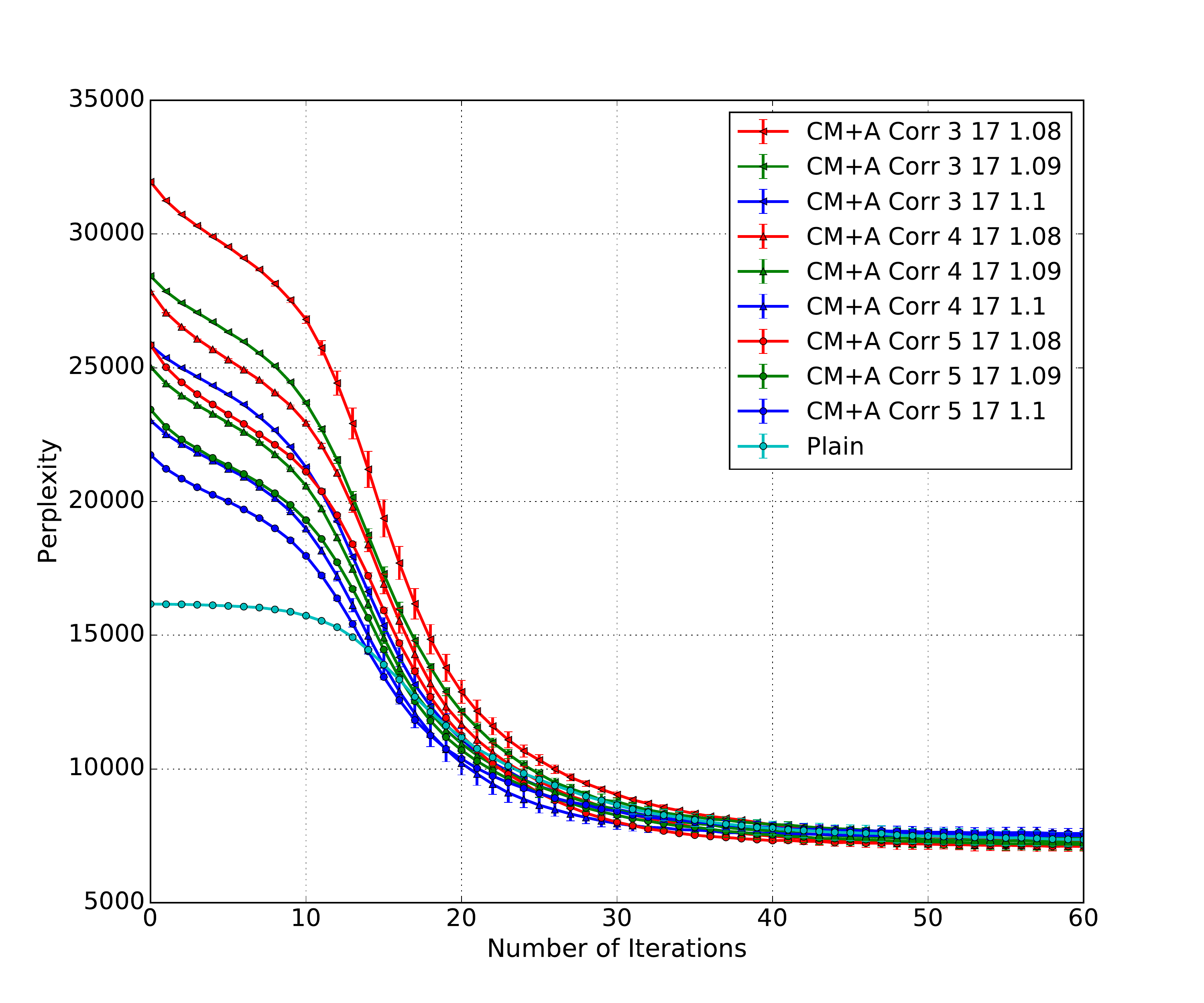}
  }
  \subfigure[18 bits  \label{fig:combo_s_18}]{
    \includegraphics[width=0.22\textwidth]{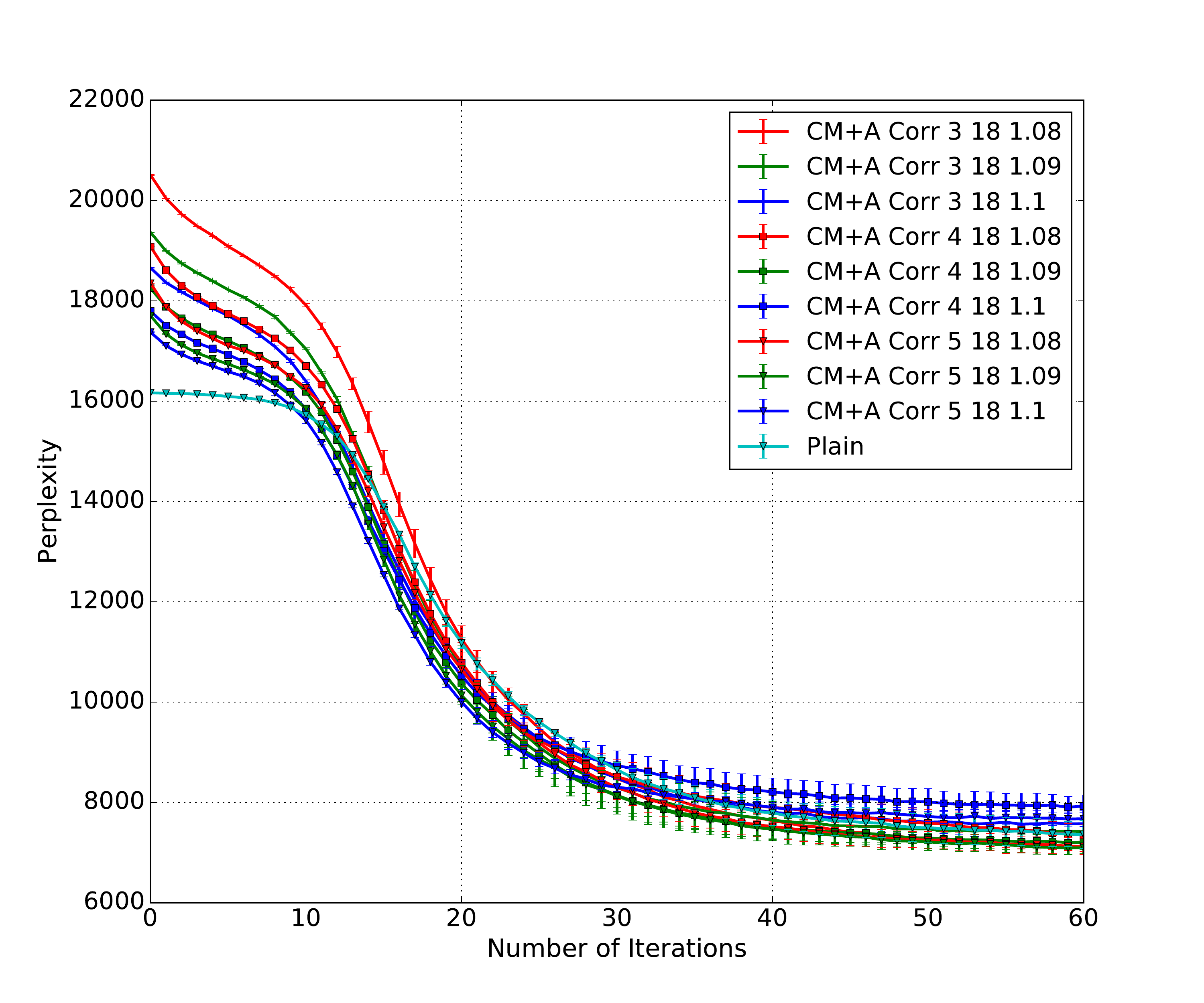}
  }
\caption{{\em Correlated counter} variant of the combined CM sketch
  and approximate counter data-structure on LDA  with a hash range of 15,16,17 and 18 as
  indicated in the respective captions.}
\label{fig:combined_s}
\end{center}
\end{figure*}

\begin{figure*}[tb]
\begin{center}
  \subfigure[15 bits \label{fig:combo_f_15}]{
    \includegraphics[width=0.22\textwidth]{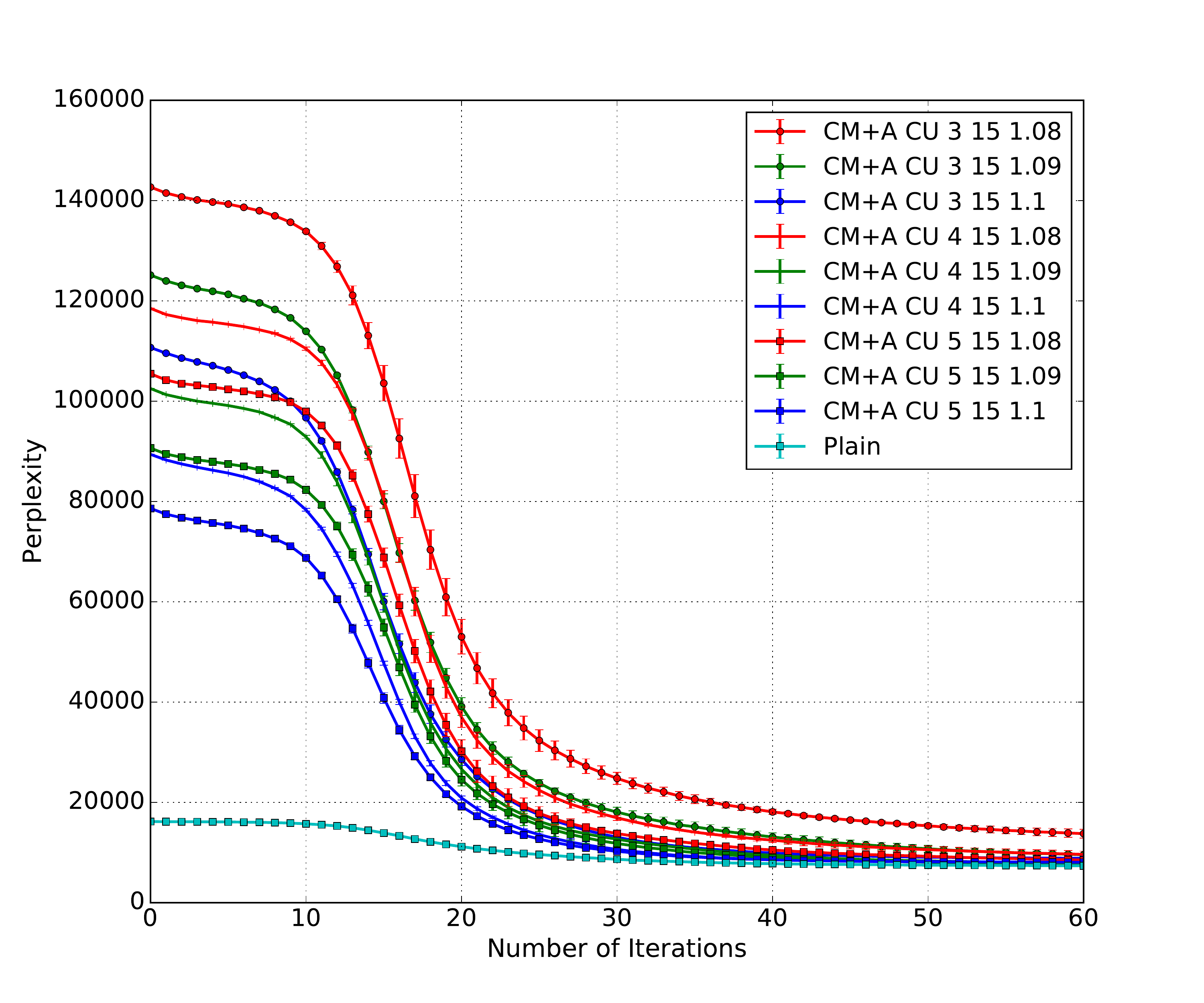}
   }
  \subfigure[16 bits \label{fig:combo_f_16}]{
    \includegraphics[width=0.22\textwidth]{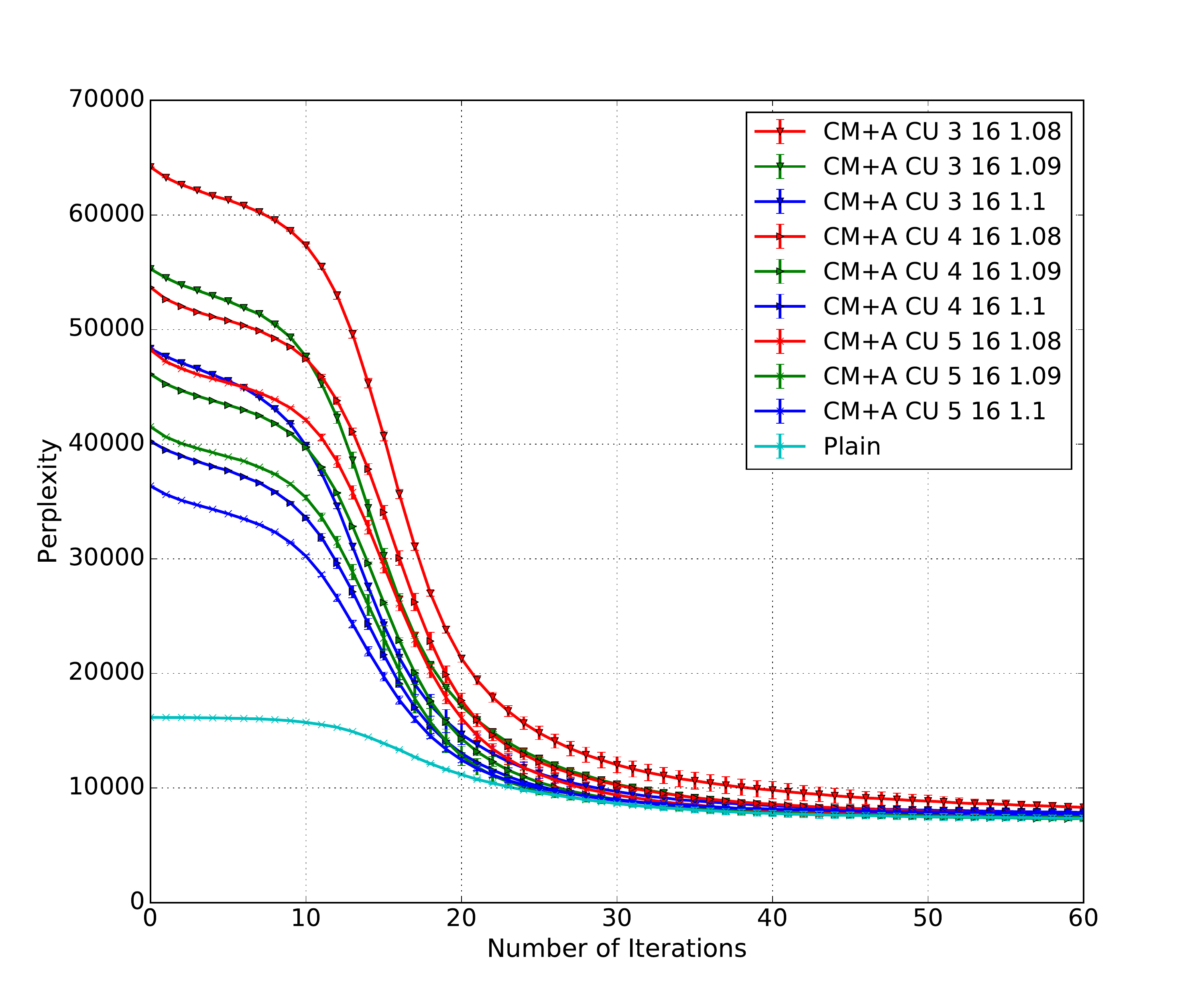}
  }
  \subfigure[17 bits \label{fig:combo_f_17}]{
    \includegraphics[width=0.22\textwidth]{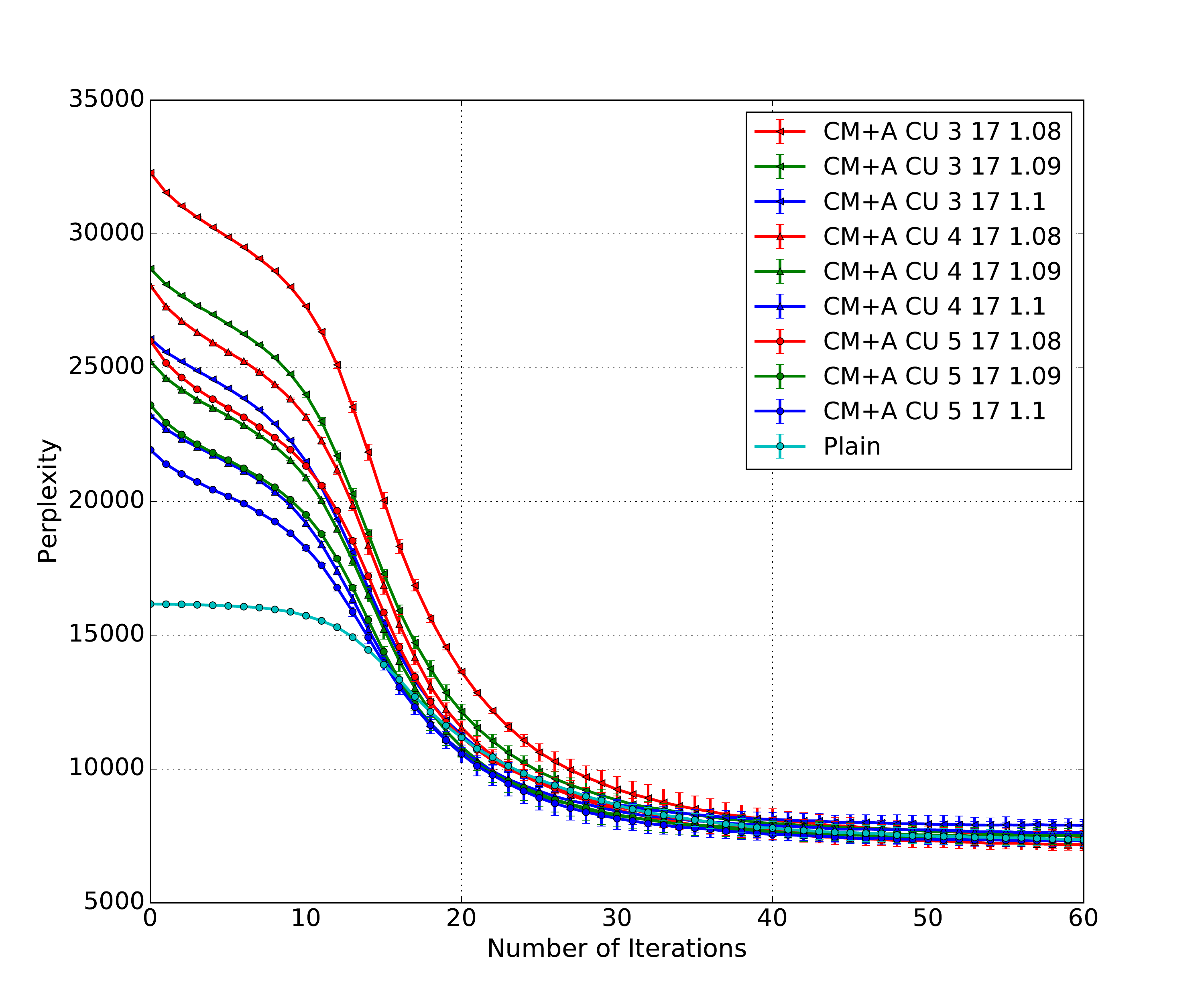}
  }
  \subfigure[18 bits  \label{fig:combo_f_18}]{
    \includegraphics[width=0.22\textwidth]{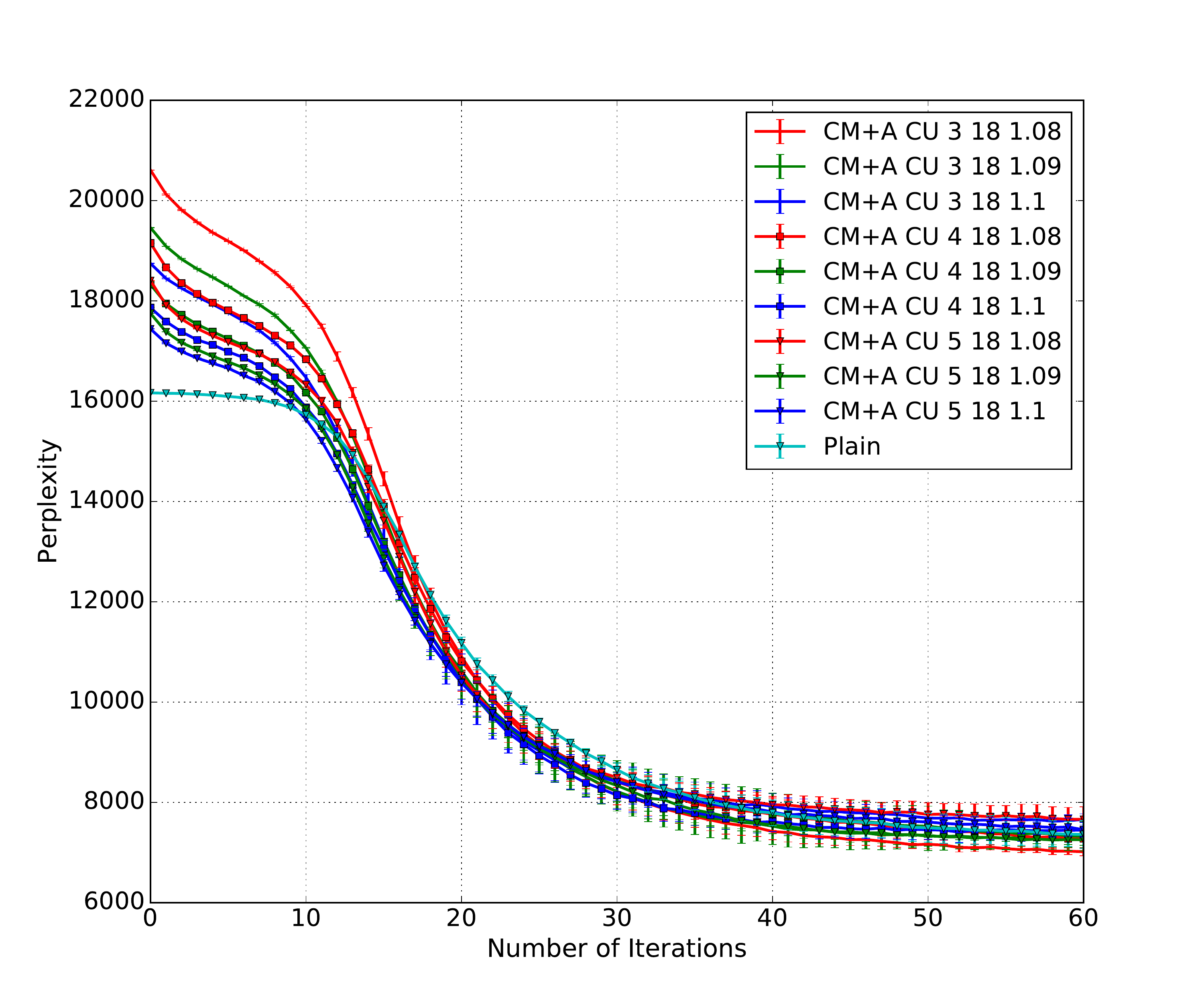}
  }
\caption{{\em Correlated counter + conservative update} variant of the combined CM sketch
  and approximate counter data-structure on LDA  with a hash range of 15,16,17 and 18 as
  indicated in the respective captions.}
\label{fig:combined_f}
\end{center}
\end{figure*}
 
\end{document}